\documentclass{article}

\PassOptionsToPackage{round}{natbib}

\usepackage[final]{neurips_2020}

\usepackage{neurips_2020}
\usepackage{caption}
\usepackage{wrapfig}
\usepackage[utf8]{inputenc} 
\usepackage{comment}
\usepackage[T1]{fontenc}    
\usepackage{hyperref}       
\usepackage{url}            
\usepackage{bm}
\usepackage{xcolor}
\definecolor{linkcolor}{RGB}{83,83,182}
\definecolor{citecolor}{RGB}{0, 12, 226}
\hypersetup{
    colorlinks=true,
    citecolor=citecolor,
    linkcolor=linkcolor,
    urlcolor=linkcolor
}
\usepackage{booktabs}       
\usepackage{amsfonts}       
\usepackage{amsthm}
\usepackage{amssymb}
\usepackage{nicefrac}       
\usepackage{microtype}      

\usepackage{algorithm}
\usepackage{algorithmic}
\usepackage[titlenumbered,ruled,noend,algo2e]{algorithm2e}

\SetCommentSty{mycommfont}
\SetEndCharOfAlgoLine{}

\usepackage{amsmath}

\usepackage{shortcuts_jac}

\graphicspath{./figures}

\definecolor{darkspringgreen}{rgb}{0.09, 0.45, 0.27}
\title{Statistical control for spatio-temporal MEG/EEG source imaging with desparsified multi-task Lasso}

\author{%
  Jerome-Alexis~Chevalier\\
  Inria Saclay\\
  Paris-Saclay, France\\
  \texttt{jerome-alexis.chevalier@inria.fr} \\
  \And
  Alexandre, Gramfort\\
  Inria Saclay\\
  Paris-Saclay, France\\
  \texttt{alexandre.gramfort@inria.fr} \\
  \AND
  Joseph Salmon \\
  IMAG, Université de Montpellier\\
  Montpellier, France \\
  \texttt{joseph.salmon@umontpellier.fr} \\
  \And
  Bertrand, Thirion\\
  Inria Saclay, CEA\\
  Paris-Saclay, France\\
  \texttt{bertrand.thirion@inria.fr} \\
}

\usepackage[capitalise,nameinlink]{cleveref}
\Crefformat{section}{#2Section~#1#3}
\Crefformat{subsection}{#2Sec.~#1#3}
\Crefformat{paragraph}{#2Sec.~#1#3}
\Crefformat{prop}{#2Prop.~#1#3}
\Crefformat{th}{#2Th.~#1#3}
\Crefformat{lem}{#2Lem.~#1#3}
\Crefformat{df}{#2Def.~#1#3}
\Crefformat{rk}{#2Rem.~#1#3}

\crefname{prop}{Proposition}{propositions}
\crefname{prop}{Proposition}{Propositions}
\crefname{lem}{lemma}{lemmas}
\Crefname{lem}{Lemma}{Lemmas}
\crefname{thm}{theorem}{theorems}
\Crefname{thm}{Theorem}{Theorems}
\crefname{df}{Definition}{definitions}
\crefname{df}{Definition}{Definitions}
\crefname{rk}{remark}{remarks}
\crefname{rk}{Remark}{Remarks}

\begin{document}

\maketitle

\begin{abstract}

  Detecting where and when brain regions activate in a cognitive task
  or in a given clinical condition is the promise of non-invasive
  techniques like magnetoencephalography (MEG) or
  electroencephalography (EEG).
  This problem, referred to as source localization, or source
  imaging, poses however a high-dimensional statistical inference
  challenge.
  While sparsity promoting regularizations have been proposed to
  address the regression problem, it remains unclear how to ensure
  statistical control of false detections.
  Moreover, M/EEG source imaging requires to work with spatio-temporal data
  and autocorrelated noise.
  To deal with this, we adapt the desparsified Lasso estimator ---an estimator
  tailored for high dimensional linear model that asymptotically follows a
  Gaussian distribution under sparsity and moderate feature correlation
  assumptions--- to temporal data corrupted with autocorrelated noise.
  We call it the desparsified multi-task Lasso (d-MTLasso).
  We combine d-MTLasso with spatially constrained clustering to reduce
  data dimension and with ensembling to mitigate the arbitrary choice
  of clustering; the resulting estimator is called ensemble of clustered
  desparsified multi-task Lasso (ecd-MTLasso).
  With respect to the current procedures, the two advantages of
  ecd-MTLasso are that \textit{i)}it offers statistical guarantees and
  \textit{ii)}it allows to trade spatial specificity for sensitivity,
  leading to a powerful adaptive method.
  %
  %
  Extensive simulations on realistic head geometries, as well as
  empirical results on various MEG datasets, demonstrate the high
  recovery performance of ecd-MTLasso and its primary practical benefit:
  offer a statistically principled way to threshold MEG/EEG source maps.
\end{abstract}

\section{Introduction}
\label{sec:intro}

Source imaging with magnetoencephalography (MEG) and
electroencephalography (EEG) delivers insights into brain
activity with high temporal and good spatial resolution
in a non-invasive way~\citep{baillet-etal:2001}.
It however requires to solve the bioelectromagnetic inverse problem,
which is a high-dimensional ill-posed regression problem.
Various approaches have been proposed to regularize the estimation of
the regression coefficients that map activity to brain locations.
Historically, $\ell_2$ regularization was considered first
 \citep{Hamalainen1994}, with successive improvements known as
dSPM~\citep{dspm} and sLORETA~\citep{PascualMarqui:2002} that are
referred to as ``noise normalized'' solutions.
The reason is that the coefficients are standardized with an estimate
of the noise standard deviation, producing outputs that are comparable
to T or F statistics, yet not statistically calibrated.
These latter techniques have since become standard when
using $\ell_2$ approaches.

More recently, alternative approaches based on sparsity assumptions
have been proposed with the ambition to improve the spatial specificity
of M/EEG source imaging~\citep{Matsuura-Okabe:1995,
Haufe_Nikulin_Ziehe_Mueller_Nolte09,gramfort-etal:2012,Lucka-etal:2012,
Wipf-Nagarajan:2009}.
The output of such methods consists of focal sources as opposed
to blurred images obtained with $\ell_2$ regularization.
However, obtaining statistics (``noise normalized'') from sparse or
non-linear estimators seems challenging, especially since M/EEG data
are spatio-temporal data with complex noise structure.
A natural way to deal with the temporal dimension is to consider a
multi-task estimator and structured sparse priors based on $\ell_1/\ell_2$
mixed norms~\citep{Ou_Hamalainen_Golland09,gramfort-etal:2012}.

In the statistical literature, some attempts to obtain an
estimate of both regression coefficients and their variance
have been proposed for linear models in high dimension
\citep{Wasserman_Roeder09, Meinshausen_Meier_Buhlmann09, Buhlmann13}.
These estimates can then be translated to $p$-value maps, \ie maps of $p$-values associated with each covariate.
Some methods adapted for sparse scenarios have then proposed
to debias the Lasso to obtain $p$-values or confidence intervals
\citep{Zhang_Zhang14,vandeGeer_Buhlmann_Ritov_Dezeure14,Javanmard_Montanari14}.
We refer to such variants as desparsified Lasso.
Recently, desparsified extensions of group Lasso have also been considered \citep{Mitra_Zhang16,Stucky_vandeGeer18}.
However, all these previous methods generally lack of power when $p \gg n$.
Here, we propose to address a multi-task setting in the presence of
correlated noise, and to deal with high-dimensional when $p \gg n$
leveraging on data structure as done by \citet{Chevalier_Salmon_Thirion18}.
All these challenges need to be considered for M/EEG source imaging.

Our first contribution is to propose the desparsified
multi-task Lasso (\dMTLasso), an extension of the desparsified Lasso
(\dLasso) \citep{Zhang_Zhang14,vandeGeer_Buhlmann_Ritov_Dezeure14} to
multi-task setting \citep{Obozinski_Taskar_Jordan10}.
More precisely, we adapt the group formulation by \citet{Mitra_Zhang16}
to the multi-task setting that enjoys
\textit{i)}a simple statistic test formula with
\textit{ii)}a natural integration of auto-correlated noise and
\textit{iii)}a simplification of the assumptions.
Our second contribution is to introduce ensemble of clustered
desparsified multi-task Lasso (ecd-MTLasso), which has two advantages
compared to current methods: \textit{i)}it offers statistical guarantees
and \textit{ii)}it allows to trade spatial specificity for sensitivity,
leading to a powerful adaptive method.
Our third contribution is an empirical validation of the theoretical
claims.
In particular, we run extensive simulations on realistic head geometries,
as well as empirical results on various MEG datasets to demonstrate the high
recovery performance of ecd-MTLasso and its primary practical benefit:
offer a statistically principled way to threshold MEG/EEG source maps.

%

\section{Theoretical Background}
\label{sec:background}

In this section, we give the noise model, we provide standard tools
for solving the source localization problem and, mainly, we present
three new methods with their assumptions and statistical guarantees.

\subsection{Model and notation}
\label{subsec:model}
%
For clarity, we use bold lowercase for vectors and bold uppercase for matrices.
For any positive integer $p \in \bbN^*$, we write $ [p]$ for
the set $\discset{1, \ldots, p}$.
For a vector $\bm\beta$, $\bm\beta_{j}$ refers to its $j$-th coordinate.
For a matrix $\*X \in \bbR^{n \times p}$, $\*X^{(-j)}$ refers to
matrix $\*X$ without the $j$-th column, $\*X_{i,.}$ refers to the $i$-th
row and $\*X_{.,j}$ to the $j$-th column and $\*X_{i,j}$ refers to the
element in the $i$-th row and $j$-th column.
%
%
The notation $\norm{\cdot}$ refers to the Frobenius norm for matrices and
to the standard Euclidean norm for vectors.
For a covariance matrix $\*M$, the Mahalanobis norm is denoted by
$\norm{\cdot}_{\*M^{-1}}$ and for a given vector $\*a$ we have
$\norm{\*a}^2_{\*M^{-1}} \triangleq \mathrm{Tr}(\*a^{\top} \*M^{-1} \*a)$.
For $\*B \in \bbR^{p \times T}$,
$\normin{\*B}_{2,1} = \sum_{j=1}^{p} \normin{\*B_{j, .}}$, and its (row)
support is $\mathrm{Supp}(\*B) = \{j \in [p] : \*B_{j,.} \neq 0\}$.
We assume that the underlying model is linear:
\begin{align}\label{eq:noise_model}
\*Y = \*X\*B + \*E \enspace ,
\end{align}
where
$\*Y \in \bbR^{n \times T}$ is the signal observed on M/EEG sensors,
$ \*X \in \bbR^{n \times p}$ the design matrix representing the M/EEG
forward model, $\*B \in \bbR^{p \times T}$ the underlying signal in
source space and $ \*E \in \bbR^{n \times T}$ the noise.
We assume that there exist $\rho \in [0,1)$ and $\sigma>0$ such that all
$t \in [T]$, $ \*E_{.,t} \sim \mathcal{N}(\*0,\sigma^2 \*I_n) $ and
that for all $i \in [n]$ and all $t \in [T -1]$,
%
$\Cor(\*E_{i,t}, \*E_{i,t+1}) = \rho.$
%
For all $i \in [n]$, $\*E_{i,.}$ is Gaussian with Toeplitz covariance,
\ie defining $\*M \in \bbR^{T \times T}$ by $\*M_{t,u} = \sigma^2 \rho^{|t-u|}$
for all $(t,u) \in [T]^2$, we have:
\begin{align}\label{eq:noise_law_2}
\*E_{i,.} \sim \mathcal{N}(\*0,\*M) \enspace.
\end{align}
We further assume that $\*X$ has been column-wise standardized
and denote by $\hat{\bm\Sigma} \in \bbR^{p \times p}$ the
empirical covariance matrix of $\*X$, \ie
$\hat{\bm\Sigma} = \*X^{\top} \*X / n$ with $\hat{\bm\Sigma}_{j,j} = 1$.
All proofs are given in \Cref{sec:proof}.

\subsection{Metrics for statistical inference in M/EEG}
\label{subsec:metrics}
%
To quantify the ability of a M/EEG source imaging technique
to obtain a good estimated $\hat{\*B}$,
a commonly reported quantity is the Peak
Localization Error (PLE)~\citep{HAUK20111966}.
It consists in measuring the distance (in mm) along the cortical
surface between the true simulated source and the location with
maximum amplitude in the estimator.
By contrast, spatial dispersion (SD) measures how much
the activity is spread out by the inverse method~\citep{MOLINS20081069}.

To quantify the control of statistical errors,
we consider a generalization of the Family Wise Error Rate (FWER) \citep{Hochberg_Tamhane87}: the $\delta$-FWER.
As illustrated in \Cref{fig:spatial_tolerance} in appendix,
it is the FWER taken with respect to a ground truth dilated spatially
by an amount $\delta$ ---in the present study a distance in mm.
A rigorous definition of $\delta$-FWER is given in \Cref{sec:complement_metrics}.
The rationale is that detections made outside of the support, but less
than $\delta$ away from the support should count as slight
inaccuracies of the methods, not as false positives.
In an analogous manner, $\delta\text{-FDR} = (1 -
\delta\text{-precision})$ has been proposed recently as an extension of the
False Discovery Rate (FDR) \citep{Benjamini_Hochberg95} to include a
spatial tolerance~\citep{Nguyen_Chevalier_Thirion19,Gimenez_Zou19}.
We thus characterize the selection capabilities of the methods
through a $\delta$-precision/recall curve.

\subsection{Classical Solutions}
\label{subsec:competitors}
%
The sLORETA and dSPM estimators are derived from the ridge estimator \citep{Hoerl_Kennard70}:
\begin{align}\label{eq:beta_ridge}
  \hat{\*B}^{\rm{Ridge}} = \*K \*Y
  \quad\textrm{where}\quad
  \*K = \*X^\top (\*X \*X^\top + \lambda \*I)^{-1}
  \enspace .
\end{align}
They are obtained by scaling each row $j$ in $\hat{\*B}^{\rm{Ridge}}$ by
an estimate of the noise level at location $j$. It reads~\citep{Lin:2006}
$\hat{\*B}^{\rm{dSPM}}_{j,t} = \hat{\*B}^{\rm{Ridge}}_{j,t} / \sigma^{\rm{dSPM}}_j$
and $\hat{\*B}^{\rm{sLORETA}}_{j,t} = \hat{\*B}^{\rm{Ridge}}_{j,t} / \sigma^{\rm{sLORETA}}_j$,
where $\sigma^{\rm{dSPM}}_j = \sqrt{\sigma^2 [\*K\*K^\top]_{j,j}}$
and $\sigma^{\rm{sLORETA}}_j = \sqrt{[\*K(\sigma^2 \*I + \*X\*X^\top)\*K^\top]_{j,j}}$.
Interestingly, it can be proved that in the absence of noise and when only
a single coefficient is non-zero, the sLORETA estimate has its maximum
at the correct location~\citep{PascualMarqui:2002}.
Assuming $\*B_{.,t} \sim \mathcal{N}(\*0,\*I)$, the covariance of $\*Y$
reads $\sigma^2 \*I + \*X\*X^\top$.
Hence, one can consider that sLORETA adds to dSPM an extra term in the sensor covariance matrix that comes from the sources.
Note that these methods treat each time instant independently, hence ignoring source and noise temporal autocorrelations.

\subsection{Desparsified multi-task Lasso (\dMTLasso)}
\label{sub:desparsified_multi_task_lasso}
%
Let us first recall the definition of the multi-task Lasso (MTLasso) estimator \citep{Obozinski_Taskar_Jordan10} in our setting.
For a tuning parameter\footnote{$\lambda$ is set by cross-validation on a logarithmic grid going from $\tfrac{\lambda_{\max}}{100}$ to $\lambda_{\max}$, where $\lambda_{\max}= \norm{\*X^\top \*Y}_{2,\infty}$.} $\lambda>0$, it is defined as
\begin{align}\label{eq:beta_mtl}
  \hat{\*B}^{\rm{MTL}} \in \underset{\*{B} \in \bbR^{p \times T}}{\argmin}
\left\{
\frac{1}{2n}\norm{\*Y - \*X \*B}^{2} + \lambda \norm{\*B}_{2,1}
\right\} \enspace.
\end{align}
It is well known that similarly to the Lasso, MTLasso is biased: it tends to shrink rows with large amplitude towards zero.
Below, we provide an adaptation of the Desparsified Lasso following the approach by \citet{Zhang_Zhang14}, see also \cite{Mitra_Zhang16}, to ensure statistical control.
The approach relies on the introduction of score vectors $\*z_1,\dots,\*z_p$ in $\bbR^n$ defined by
\begin{equation}\label{eq:z-def}
    \*z_{j} = \*X_{\cdot,j} - \*X^{(-j)} \hat{\bm\beta}^{(-j)}_{\bm\alpha_j} \enspace ,
\end{equation}
where, for $j\in [p]$, $\hat{\bm\beta}^{(-j)}_{\bm\alpha_j}$ is the
Lasso solution (\citet{tibshirani1996, chen1994basis}) of the regression
of $\*X_{\cdot,j}$ against $\*X^{(-j)}$ with regularization parameter\footnote{In \citep[Table 1]{Zhang_Zhang14} an algorithm for choosing $\bm\alpha_j$ is proposed.
We noticed that taking for all $j \in [p]$, $\bm\alpha_{j} = c \bm\alpha_{\max, j} := c \normin{\*X^{(-j)} \*X_{\cdot,j}}_{\infty} / n$
with $c = 0.5\%$ for M/EEG data allows to make a
significant computation gain and yields adequate residuals for $C=1000$
(see \Cref{sub:clustering_to_handle_structured_high_dimensional_data_}).}
$\bm\alpha_j$.
Note that these score vectors are independent of $\*Y$ and
their computation is then equivalent to solving the node-wise
Lasso \citep{Meinshausen_Buhlmann06}.
For such vectors, the noise model in \eqref{eq:noise_model} yields
\begin{align}\label{eq:debiasied_explained}
	\frac{\*z_{j}^\top \*Y}{\*z_{j}^\top \*X_{.,j}}
	= \*B_{j,.}
	  + \frac{\*z_{j}^\top \*E}{\*z_{j}^\top \*X_{.,j}}
	  + \sum_{k \neq j} \frac{\*z_{j}^\top \*X_{.,k} \*B_{k,.}}{\*z_{j}^\top \*X_{.,j}} \enspace.
\end{align}
Discarding the noise term and plugging $\hat{\*B}_{k, .}^{\rm{MTL}}$ as
a preliminary estimator of $\*B_{k,.}$ in \eqref{eq:debiasied_explained},
we coin the desparsified multi-task Lasso (\dMTLasso), a debiased
estimator of $\hat{\*B}^{\rm{MTL}}$ defined for all $j\in [p]$ by
\begin{align}\label{eq:debiasied_MTL}
	\hat{\*B}_{j,.}^{(\rm\dMTLasso)}
	=
	\frac{\*z_{j}^\top \*Y}{\*z_{j}^\top \*X_{.,j}}
	- \sum_{k \neq j} \frac{\*z_{j}^\top \*X_{.,k} \hat{\*B}^{\rm{MTL}}_{k,.}}{\*z_{j}^\top \*X_{.,j}}\enspace .
\end{align}
To derive \dMTLasso statistical properties, we need
the extended Restricted Eigenvalue (RE) property
\citep[Assumption 3.1]{Lounici_Pontil_vandeGeer_Tsybakov11},
detailed in \Cref{sec:RE_assumption}. More precisely, we assume that

(A1) RE($\*X,s$) is verified on $\*X$ for a sparsity parameter $s \geq |\mathrm{Supp}(\*B)|$ and a constant $\kappa = \kappa(s) > 0$.

Roughly, A1 can be seen as a combination of sparsity and "moderate" feature correlation assumptions.
\begin{prop}\label{prop:desparsified_mtlasso}
Considering the model in \Cref{eq:noise_model}, assuming A1
and for a choice of $\lambda$ large enough\footnote{See the proof of
\citep[Theorem3.1]{Lounici_Pontil_vandeGeer_Tsybakov11}.} in \Cref{eq:beta_mtl}, then with high probability:
\begin{align}
\label{eq:delta}
& \sqrt{n}(\hat{\*B}^{{(\rm \dMTLasso)}} - \*B) = \bm\Lambda + \bm\Delta
 \enspace , \\
& \bm\Lambda_{j, .} \sim
\mathcal{N}_{p}(\*0, \, \hat{\bm\Omega}_{j, j} \*M)
, \text{ for all } j \in [p], \enspace  \text{where} \enspace
\hat{\bm\Omega}_{j, k} = \frac{n \*z_{j}^\top \*z_{k}}
    {|\*z_{j}^\top \*X_{.,j}| |\*z_{k}^\top \*X_{.,k}|}
\nonumber
    \\
& \norm{\bm\Delta}_{2, 1} =
\mathrm{O} \left(\frac{s \lambda \sqrt{\log(p)}}{\kappa^2}\right)
\end{align}
\end{prop}

Then, under the $j$-th null hypothesis $H^{(j)}_0$ : ``$\*B_{j,.}=0$''
and neglecting the term $\bm\Delta$
(see \Cref{sub:proof_of_prop:desparsified_mtlasso} for more details)
in \eqref{eq:delta} as done by \citet{vandeGeer_Buhlmann_Ritov_Dezeure14},
$\hat{\*B}_{j,.}^{(\rm\dMTLasso)}$ is Gaussian with zero-mean.
Finally, using standard results on $\chi^2$ distributions (see
\Cref{sub:probability_lemma}), we obtain
\begin{align*}
    n \norm{\hat{\*B}_{j,.}^{(\rm \dMTLasso)}}_{\*M^{-1}}^2
    \sim \hat{\bm\Omega}_{j, j} \chi^2_T \enspace.
\end{align*}
If $\*M$ is known, the quantity
$n \normin{\hat{\*B}_{j,.}^{(\rm \dMTLasso)}}_{\*M^{-1}}^2  / \hat{\bm\Omega}_{j, j}$
can be used as a decision statistic to obtain
a $p$-value testing the importance of source $j$
by comparison with the $\chi^2_T$ distribution.
In practice we need to estimate $\*M$ by $\hat{\*M}$.  Notably,
assuming that we have an estimator $\hat{\sigma}$ of $\sigma$ that verifies
approximately $(n - \hat{s})\hat{\sigma}^2/\sigma^2 \sim \chi^2_{n  - \hat{s}}$,
where $\hat{s}=|\mathrm{Supp}(\hat{\*B}^{\rm{MTL}})|$ (see
\Cref{subsec:method}), we take
\begin{align}\label{eq:inference_stat}
	\hat{f}_j:=
  \frac{n \normin{\hat{\*B}_{j,.}^{(\rm \dMTLasso)}}_{\hat{\*M}^{-1}}^2}
  {T \, \hat{\bm\Omega}_{j, j}}\enspace,
\end{align}
as statistic to compare with a Fisher distribution with parameters $T$ and $n - \hat{s}$, to compute the $p$-values.
The full \dMTLasso algorithm is given in \Cref{alg:dMTLasso}.
Note that, a Python implementation of the procedures presented in this paper
is available on \url{https://github.com/ja-che/hidimstat} along with some
examples.
{\fontsize{4}{4}\selectfont
\begin{algorithm}[t]
\SetKwInOut{Input}{input}
\SetKwInOut{Init}{init}
\SetKwInOut{Parameter}{param}
\caption{{\dMTLasso}
}
\Input{$
    \*X \in \bbR^{n \times p}, \*Y$}

    $\hat{\*B}^{\rm{MTL}} \leftarrow {\rm MTL}(\*X, \*Y) $
    \tcp*{cross-validated multi-task Lasso}

    $\hat{\*E} \leftarrow \*Y-\*X\hat{\*B}^{\rm{MTL}}$
    \tcp*{Residuals}

	$\hat{s} \leftarrow |\mathrm{Supp}(\hat{\*B}^{\rm{MTL}})|$

	\For(\tcp*[f]{Noise level estimation}){ $t\in[T]$}
	{
	$\hat{\sigma}^2_t = \normin{\hat{\*E}_{.,t}}^2 / (n - \hat{s})$
	}

    $\hat{\sigma}^2 = \mathrm{median}(\{\hat{\sigma}_t^2, t \in [T]\})$

    Get $\hat{\*M} $ thanks to \Cref{subsec:method}

	\For{ $j\in[p]$}{

	$\*z_j \leftarrow {\rm Lasso}(\*X^{(-j)}, \*X_{.,j})$
    \tcp*{cross-validated Lasso}

	$\hat{\bm\Omega}_{j, j} \leftarrow \frac{n \*z_{j}^\top \*z_{j}} {|\*z_{j}^\top \*X_{.,j}| |\*z_{j}^\top \*X_{.,j}|}$

	$\hat{\*B}_{j,.}^{(\rm\dMTLasso)}
	\leftarrow
	\frac{\*z_{j}^\top \*Y}{\*z_{j}^\top \*X_{.,j}}
	- \sum_{k \neq j} \frac{\*z_{j}^\top \*X_{.,k} \hat{\*B}^{\rm{MTL}}_{k,.}}{\*z_{j}^\top \*X_{.,j}}$
    \tcp*{Desparsified multi-task Lasso}

	$\hat{f}_j \leftarrow \frac{n \normin{\hat{\*B}_{j,.}^{(\rm \dMTLasso)}}_{\hat{\*M}^{-1}}^2} {T \, \hat{\bm\Omega}_{j, j}}$
    \tcp*{Inference statistics}
	}

\Return{$\hat{f}_1,\dots,\hat{f}_p$}
\label{alg:dMTLasso}
\end{algorithm}

}
%
\subsection{Noise parameters estimation}
\label{subsec:method}

In \Cref{subsec:model} noise is assumed homogeneous across sensors, allowing to obtain a robust estimator.
Extending \citet{Reid_Tibshirani_Friedman16} to multi-task regression, we consider the residuals $\hat{\*E} = \*Y - \*X \hat{\*B}^{\rm{MTL}}$,
and the estimated support size $\hat{s}$.
Defining, for $t \in [T]$,
$\hat{\sigma}^2_t = \normin{\hat{\*E}_{.,t}}^2 / (n - \hat{s})$,
an estimate of $\sigma^2$ is:
\begin{align*}
    \hat{\sigma}^2 = \mathrm{median}(\{\hat{\sigma}^2_t, t \in [T]\})\enspace.
\end{align*}
Taking the median instead of the mean avoids depending on prospective
under-fitted time steps and turns out to be more robust empirically.
Similarly, defining for all $t \in [T - 1]$,
$\hat{\rho}_t = \cor_n(\hat{\*E}_{.,t}, \hat{\*E}_{.,t+1})$ (where $\cor_n(., .)$ is the empirical correlation), $\rho$ is estimated by taking
$\hat{\rho} = \mathrm{median}(\{\hat{\rho}_t, t \in [T - 1]\})$.
Then, an estimator $\hat{\*M}$ of $\*M$ is given by
$\hat{\*M}_{t,u} = \hat{\sigma}^2 \hat{\rho}^{|t-u|}$.
%
\subsection{Clustering to handle spatially structured high-dimensional data}
\label{sub:clustering_to_handle_structured_high_dimensional_data_}
%
In the high-dimensional inference scenario considered, the number of
sensors is more than one order of magnitude smaller than the number
of sources, $n \ll p$.
Therefore, estimators of conditional association between
sources and observations struggle to identify the solution.
The setting is even more difficult due to the presence
of high correlation between sources (see \Cref{fig:correlation_data} in
appendix).
Further gains can however come from a compression of the design matrix \citep{Buhlmann:2013,Mandozzi:2016}.
For this we introduce a clustering step that reduces data
dimensionality while leveraging spatial structure.
We consider a spatially-constrained hierarchical clustering
algorithm described by \citet{Varoquaux_Gramfort_Thirion12}
that uses Ward criterion\footnote{A typical
choice is $C=1000$ clusters for M/EEG data.}.
Other clustering schemes might be considered, as long as
they yield spatially contiguous regions of the cortical surface.
The combination of this clustering algorithm with the \dLasso or \dMTLasso algorithms will be respectively referred to as clustered desparsified Lasso (\cdLasso) and clustered desparsified multi-task Lasso (\cdMTLasso).

The number of clusters is denoted by $C$ and,
for $r \in [C]$, we denote by $G_r$ the $r$-th group.
Every cluster representative variable is given by the average of the
covariates it contains.
Then, reordering conveniently the columns of $\*X$, the compressed design matrix $\*Z \in \bbR^{n \times C}$ is given by:
%
\begin{align}\label{eq:clustering}
\*Z = \*X\*A \enspace , ~~
 ~~ \*A = \left[
\begin{matrix}
\frac{1}{|G_1|} & \horzbar & \frac{1}{|G_1|} & 0 & \horzbar & 0 &\ldots & 0 & \horzbar & 0\\
0 & \horzbar & 0  & \frac{1}{|G_2|} & \horzbar & \frac{1}{|G_2|} & \ldots & 0 & \horzbar & 0\\
\vdots & \vdots & \vdots & \vdots & \vdots & \vdots & \ddots & \vdots  & \vdots & \vdots \\
0 & \horzbar & 0  & 0 & \horzbar & 0 & \ldots & \frac{1}{|G_r|} & \horzbar & \frac{1}{|G_r|}\\
\end{matrix}
\right] \enspace ,
\end{align}
where $\*A \in \bbR^{p \times C}$.
%
We say that the compression of $\*X$ is of good quality if:

(A2) there exists $\bm\Gamma \in \bbR^{C \times T}$
such that $\bm\Gamma_{r, .} = \sum_{j \in G_r}{w_{j}\*B_{j, .}}$
with $w_{j} \geq 0$ for all $j \in [p]$, and the associated
compression loss $\*X \*B - \*Z\bm\Gamma$ is "small enough"
with respect to the model noise
(see \Cref{sub:proof_of_prop_cdMTL} for more details).

(A3)\footnote{$|\mathrm{Supp}(\bm\Gamma)| \leq |\mathrm{Supp}(\*B)|$
and $\*Z$ is generally better conditioned than $\*X$ making
A3 more plausible than A1.}
RE($\*Z,s^\prime$) is verified on $\*Z$ for sparsity parameter
$s^{\prime} \geq |\mathrm{Supp}(\bm\Gamma)|$ and constant
$\kappa^{\prime} = \kappa^{\prime}(s^{\prime}) > 0$.

\begin{prop}\label{prop:cdMTL}
Assume \Cref{eq:noise_model},  A2, A3,
a choice of regularization parameter in the MTLasso regression of
$\*Z$ against $\*Y$  that is large enough, and that the largest cluster of
the compression is of size $\delta$, then \cdMTLasso controls the
$\delta$-FWER.

\end{prop}
%
\subsection{Ensemble of clustered desparsified multi-task Lasso (\ecdMTLasso)}
\label{sub:ensembling_step}
%
\begin{figure}
  \centering
  \includegraphics[width=0.75\textwidth]{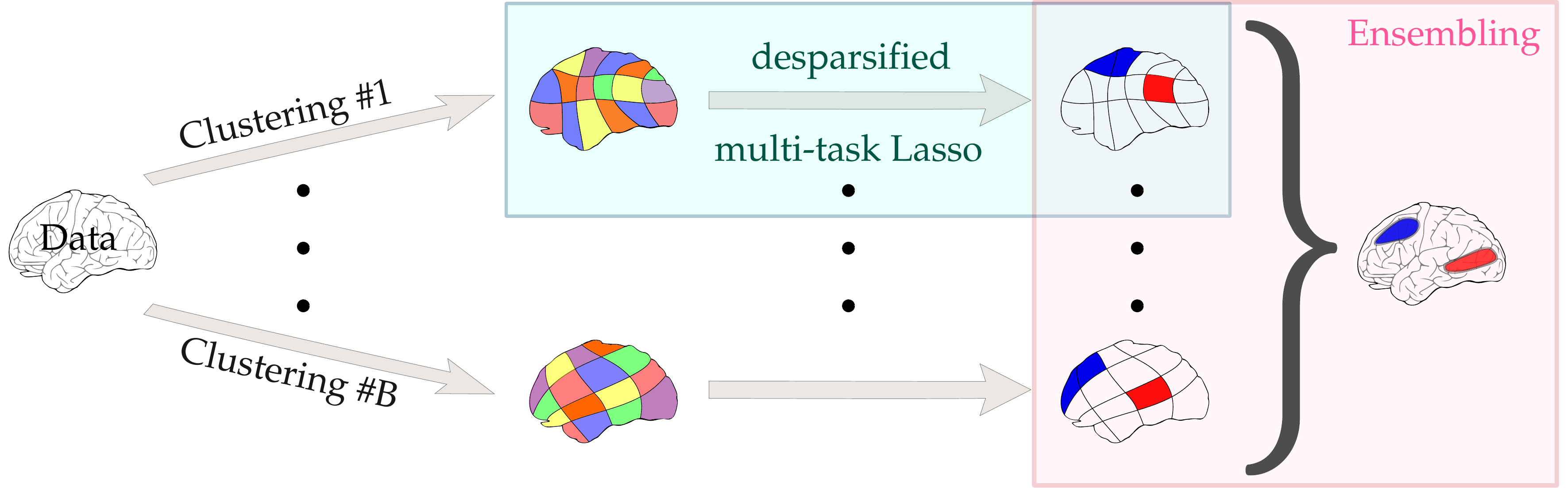}
  \caption{\textbf{ecd-MTL overview diagram.}
  While \cdMTLasso applies \dMTLasso to clustered data,
  \ecdMTLasso aggregates several \cdMTLasso solutions.
  \label{fig:ecdMTL_diagram}
  }
\end{figure}
To reduce the sensitivity of \cdMTLasso to small data perturbations,
we propose to randomize over the clustering.
We build several clustering solution, considering $B=100$ different
random subsamples of size $10\%$ of the full sample;
then we aggregate the $p$-value maps output by \cdMTLasso.
To aggregate the $B$ \cdMTLasso solutions, we use the adaptive
quantile aggregation proposed by \citet{Meinshausen_Meier_Buhlmann09}
detailed in \Cref{sec:complement}.
The full procedure of ensembling $B$ \cdMTLasso (resp. \cdLasso),
solutions is called \ecdMTLasso for ensemble of clustered desparsified
multi-task Lasso (resp. \ecdLasso).
Algorithm of \ecdMTLasso is given in \Cref{alg:ecdMTLasso} in appendix.
Also, we give an overview diagram to clarify the nesting structure of
the proposed solutions in \Cref{fig:ecdMTL_diagram}.

\begin{prop}\label{prop:ecdMTL}
Assume that for each of the $B$ compressions the hypotheses
of \Cref{prop:cdMTL} are verified, then \ecdMTLasso controls the $\delta$-FWER.
\end{prop}
This result is conservative and mixing several \cdMTLasso usually reduces
the spatial tolerance $\delta$.
Additional details on the procedure and computational complexity are
deferred to \Cref{subsec:computational}.

\section{Experiments}
\label{sec:expts}

In this section, we give empirical evidence of the advantages of
ecd-MTLasso for source localization.
First, in a typical point source simulation, we compare the methods
with respect to the standard PLE metric; notably, we study the effect of
i/clustering and ii/integrating time dimension.
In a second simulation with more realistic features,
we examine the $\delta$-FWER control property
and compare the support recovery properties of all methods.
Lastly, working on real MEG data, we show that, contrary to sLORETA,
ecd-MTLasso retrieves expected patterns using a universal threshold.

\subsection{Simulation study}
\label{subsec:sparse_simu}
%
Here, we study how the proposed estimators perform compared
to standard $\ell_2$ regularized approaches, and assess whether
time-aware statistical analysis improves upon static d-Lasso
as it is essential for M/EEG source imaging.
We use the head anatomy and the recording setup from the \emph{sample}
dataset publicly available from the MNE software~\citep{mne}.
The design matrix $\*X$ is computed with a three-shell boundary element
model with $p=7498$ candidate cortical locations, and a 306-channels
Elekta Neuromag Vectorview system with 102 magnetometers and 204 gradiometers.
We only keep the gradiometers and remove one defective sensor
leading to $n = 203$.
When considering multiple consecutive time instants to demonstrate
the ability of the solver to leverage spatio-temporal data, the source
is fixed and the temporal noise autocorrelation is set to $\rho = 0.3$.

%

%
%
\Cref{fig:metrics_hist_sparse_simu} reports the normalized histograms of PLE for
the 7498 locations for the different methods investigated; results on spatial
dispersion (SD) are available in \Cref{fig:metrics_sd_hist_sparse_simu}
in appendix.
While it might seem simplistic to consider a single source, this
experiment allows to demonstrate that d-Lasso improves over
sLORETA in the presence of noise
(see \Cref{fig:metrics_hist_sparse_simu}, left).
In the same figure, one can observe that clustering degrades this
performance, as it carries an intrinsic spatial blur.
However, even in this adversarial scenario (Dirac-like source location),
cd-Lasso and ecd-Lasso remain competitive \wrt sLORETA, avoiding
extreme PLE values.
Note that, here, a single time point was used (T=1).

The right panel in \Cref{fig:metrics_hist_sparse_simu} shows that
\dMTLasso (T=6) significantly outperforms \dLasso (T=1) in terms of PLE.
Leveraging spatio-temporal data indeed increases the signal-to-noise
ratio, which enhances spatial specificity.
Effects in terms of SD are minor (see appendix, \Cref{fig:metrics_sd_hist_sparse_simu}).
\begin{figure}
  \centering
  \includegraphics[width=0.8\textwidth]{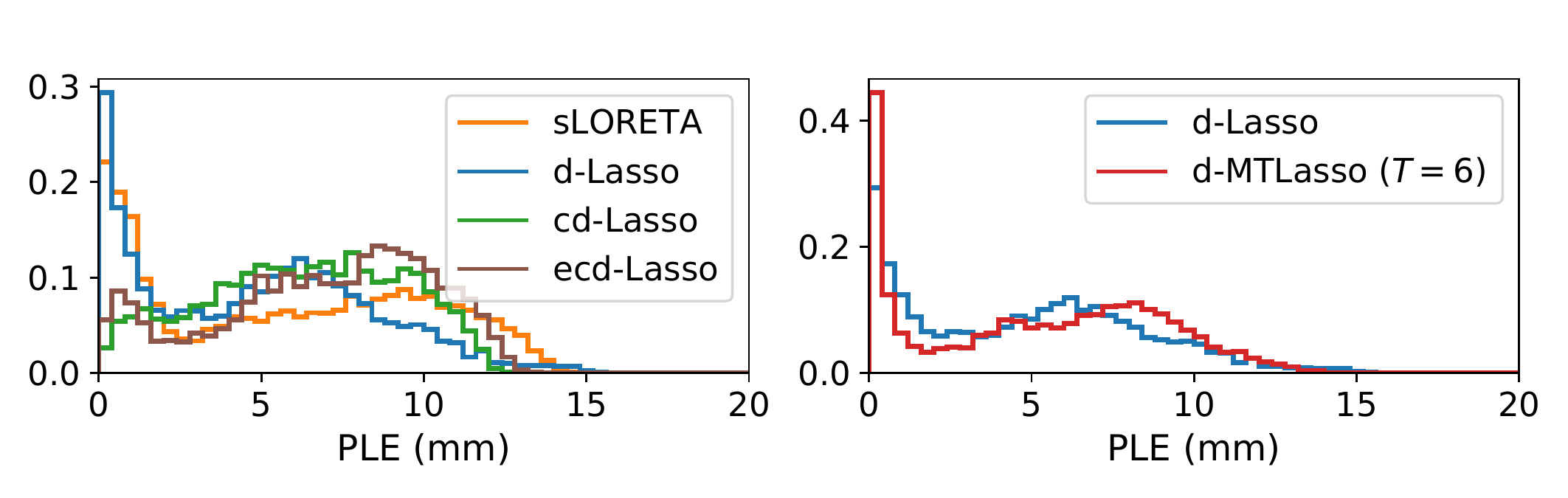}
  \caption{\textbf{Peak Localization Error (PLE) histograms.}
  (left): PLE on a fixed time point (T=1), sLORETA is outperformed by
  desparsified Lasso; cd-Lasso and ecd-Lasso are more concentrated and
  exhibit a smaller number of very low PLE but also a smaller number of
  extreme PLE values.
  (right): PLE for desparsified multi-task Lasso (d-MTLasso) with T=6 compared
  to d-Lasso (T=1). More time points improve the results by reducing the PLE.
  \label{fig:metrics_hist_sparse_simu}
  }
\end{figure}
%


%

%
\subsection{Experiments on FWER control}
\label{subsec:semi_real_simu}
%
We now investigate whether the different versions of \dMTLasso control
the $\delta$-FWER on a realistic simulation, and compare their
support recovery properties.
The data are the same as in \Cref{subsec:sparse_simu}.
To simulate the sources, we randomly draw $3$ active regions by
selecting parcels from a subdivided cortical Freesurfer parcellation
with 448 parcels \citep{KHAN201857}.
For each selected parcel we take as sources all the dipoles at
a $10$-mm geodesic distance from the center of the parcel
(around $10$ dipoles per region), fixing the amplitude at 10\,nAm.
To evaluate how the methods control the $\delta$-FWER, we perform $100$
simulations and count how often active sources
are found outside the $\delta$-dilated ground truth.
In the left panel of \Cref{fig:fwer_delta_precision_recall}, we see
that \dMTLasso does not control the $\delta$-FWER, due to the violation
of some hypotheses of proposition 1, in particular those regarding
source correlation.
However, we notice that handling noise autocorrelation
reduces the empirical $\delta$-FWER.
Using clustering, assumptions of \Cref{prop:cdMTL}
are more easily met, in
particular the conditioning of the problem is improved~\citep{MATTOUT2005356}.
Yet \cdMTLasso does not control the $\delta$-FWER for $\delta = 40$\,mm,
because the $\delta$-FWER is controlled if $\delta$ is smaller than the
largest cluster diameter, which may not hold.
Finally, randomization via \ecdMTLasso further improves FWER control.
Empirically, we observe that the $\delta$-FWER is controlled
for $\delta$ around twice the average cluster diameter.
Then, with the limitation of having a compressed design matrix well
conditioned ($C$ not too large), we can reduce the tolerance
$\delta$ by increasing $C$ (empirical support of this claim in
appendix in \Cref{fig:delta-FWER table}).
We have excluded sLORETA from this study since it
does not provide guarantees on the false discoveries.

The right panel of \Cref{fig:fwer_delta_precision_recall} shows the
$\delta$-precision recall curve of the different methods.
We first notice that \dMTLasso cannot compete with sLORETA, because the
high dimensionality of the problem makes the computation of the source
importance overly ill-posed.
\cdMTLasso improves detection accuracy, but still does not
perform as well as sLORETA.
However, adding the ensembling step, the $\delta$-precision improves
strongly, making \ecdMTLasso much better than sLORETA.
In \Cref{fig:precision_recall} in appendix, we obtain similar results
when considering the standard precision-recall curve.
\begin{figure}
  \centering
  \hfill
  \includegraphics[width=0.47\textwidth]{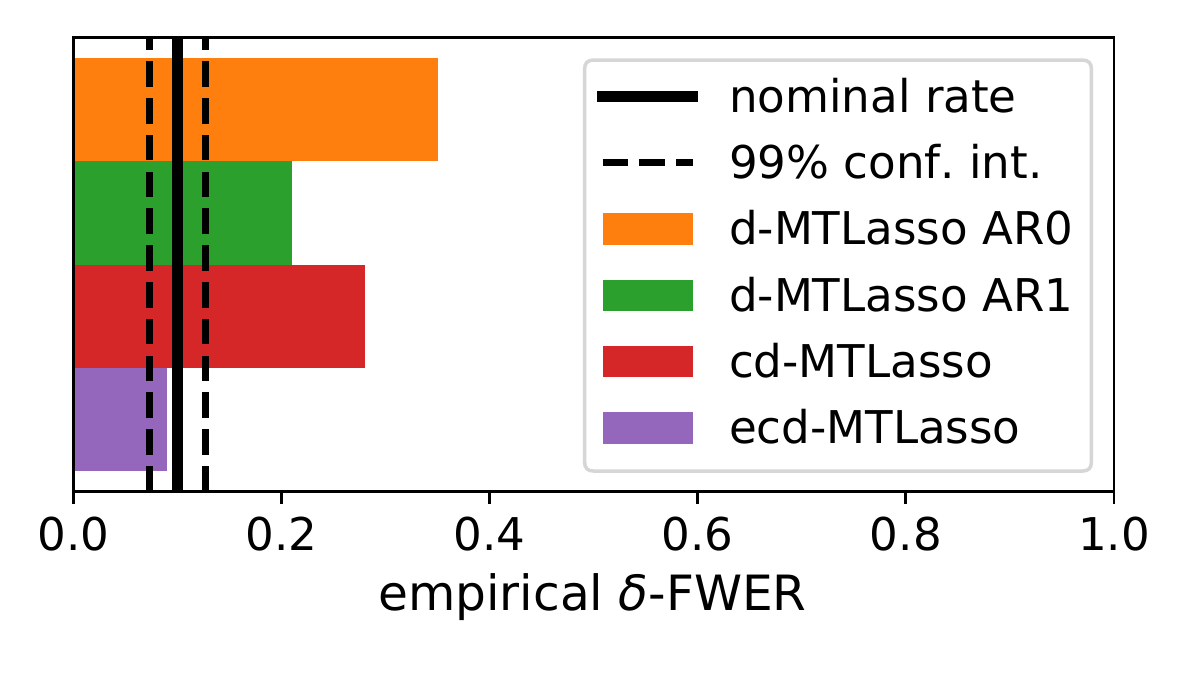}
  \hfill
  \includegraphics[width=0.47\textwidth]{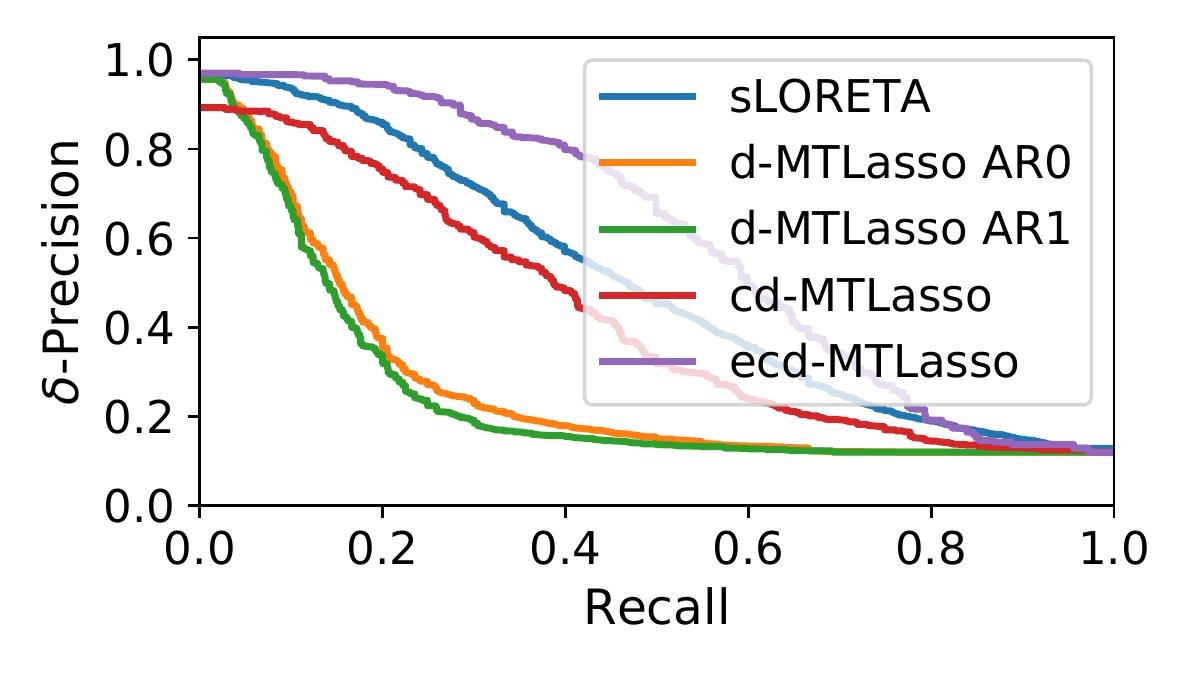}
  \hfill
  \caption{\textbf{$\delta$-FWER and $\delta$-Precision-Recall.}
  (left): $\delta$-FWER control of the different d-MTLasso methods. $\delta$-FWER control is hard for \dMTLasso and \cdMTLasso, as some detections are made far from the true sources, due to remote correlations. Ensembles of clusters allow to limit these false detections.
  (right): $\delta$-Precision-Recall curves: sLORETA outperforms d-MTLasso AR0 and AR1, because the problem is too high dimensional for the d-MTLasso to work properly. Clustering improves the outcome, and ensembling brings further benefits: \ecdMTLasso outperforms sLORETA.
  }
  \label{fig:fwer_delta_precision_recall}
\end{figure}
%
\subsection{Results on three MEG datasets}
\label{subsec:real_meg}
%
We now report results on three MEG datasets spanning three types of sensory
stimuli: auditory, visual and somatosensory.
Additional results on EEG datasets are presented in \Cref{sec:supp_fig}.
The auditory evoked fields (AEF) and visual evoked field (VEF) are obtained
using stimuli in the left ear and left visual hemifield.
The somatosensory evoked fields (SEF) are obtained following electrical
stimulation of the left median nerve on the wrist.
The detailed description of the data is provided in \Cref{sec:data}.

Experimental results are presented in
\Cref{fig:real_data_comparison} and \Cref{fig:audio_comparison}
(cf. \Cref{sec:supp_fig}).
Among the many methods for M/EEG source imaging present in the literature,
the methods that are compared here have in common to output a statistical
map.
The $\ell_2$ regularized sLORETA method is compared to the debiased
sparse estimators presented and evaluated above.
The input for all solvers is a time window of data: from $t=50$ to $t=100$\,ms
for AEF and VEF, and from $t=30$ to $t=40$\,ms for SEF.
During such time intervals one can expect the sources to originate
primarily from the early sensory cortices whose locations are
anatomically known for normal subjects.

First one can observe that all methods manage to highlight
the proper functional sensory units (planum temporale for AEF,
calcarine region for VEF and central sulcus for SEF).
Considering sLORETA results, one can observe that at a common
threshold of 3.0 on the Student statistic, the estimator is quite
spatially specific for VEF, but is overly conservative for AEF and
clearly leading to many false positives for SEF.
By inspection of the d-MTLasso solution, one can observe that taking into
account the autocorrelation of the noise leads to a better calibrated
noise variance, and therefore fewer dubious detection.
Considering ecd-MTLasso results, while all maps are also thresholded
with a single level, one can see that it retrieves expected patterns
without making dubious discoveries.

\begin{figure}
  \centering
  \includegraphics[width=0.24\textwidth]{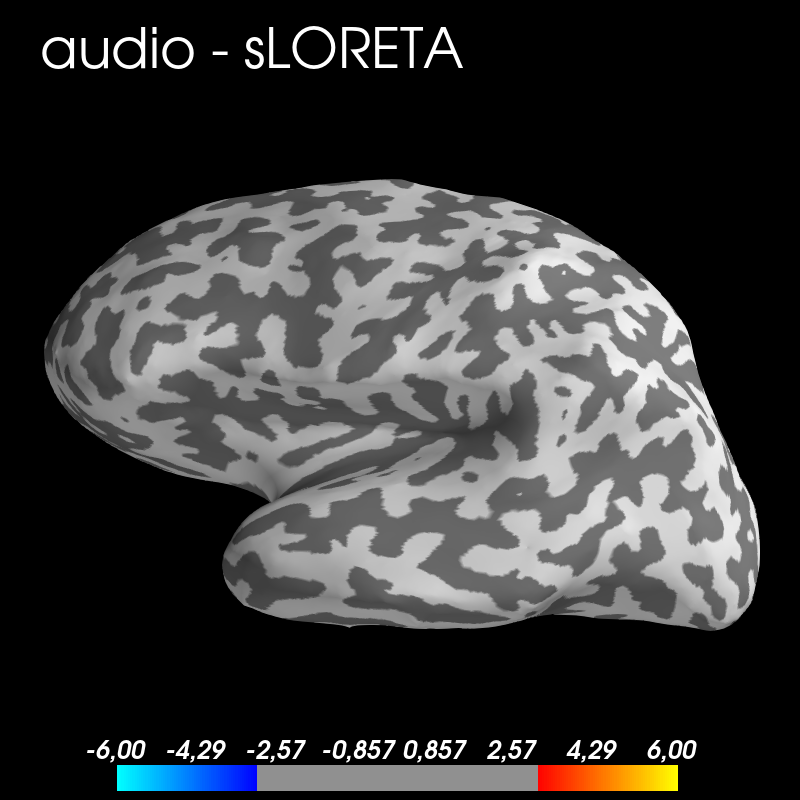}
  \includegraphics[width=0.24\textwidth]{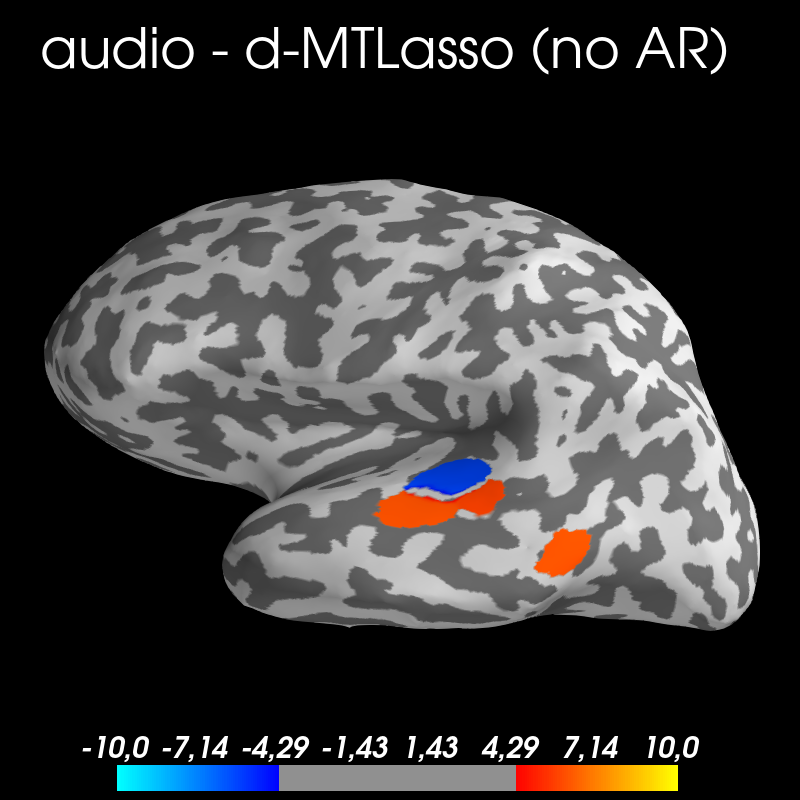}
  \includegraphics[width=0.24\textwidth]{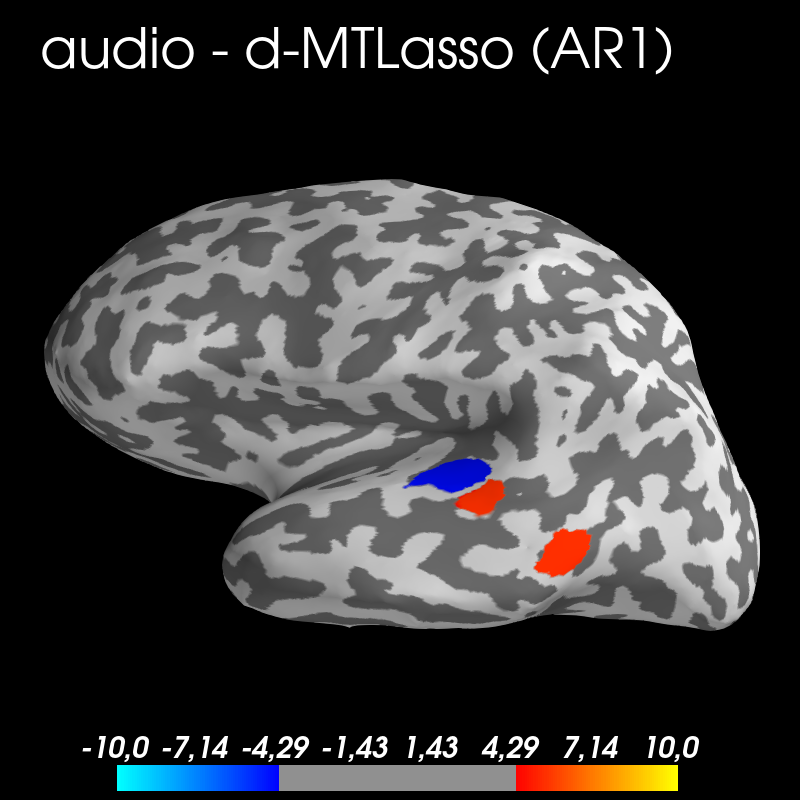}
  \includegraphics[width=0.24\textwidth]{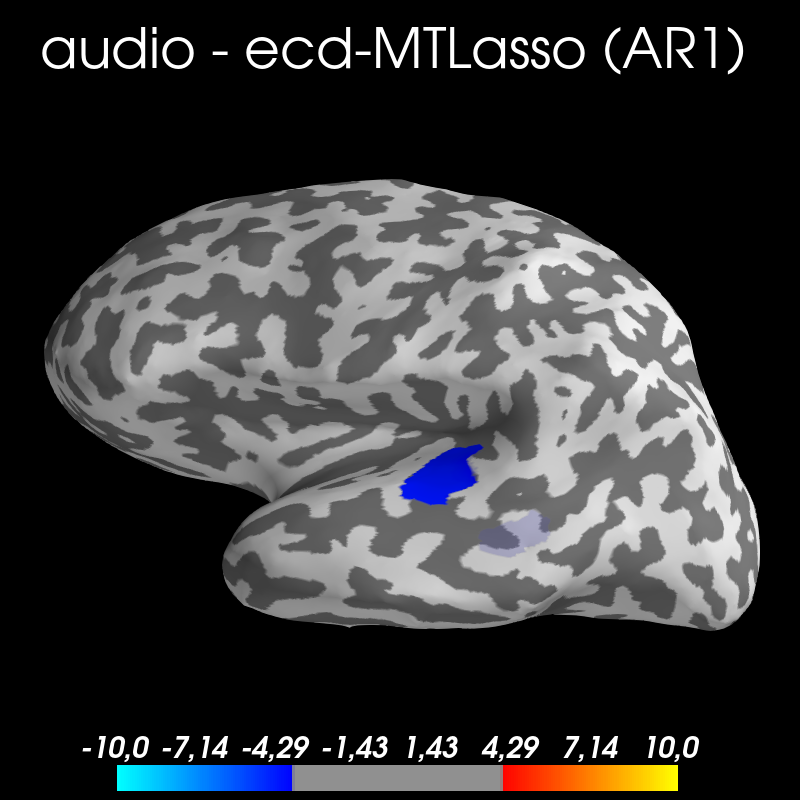}
  \includegraphics[width=0.24\textwidth]{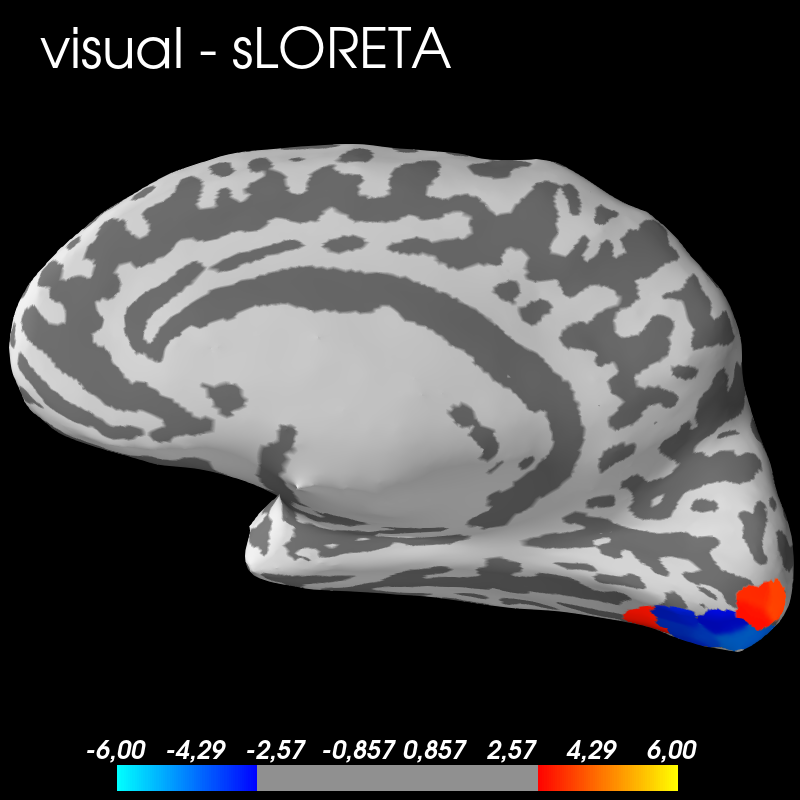}
  \includegraphics[width=0.24\textwidth]{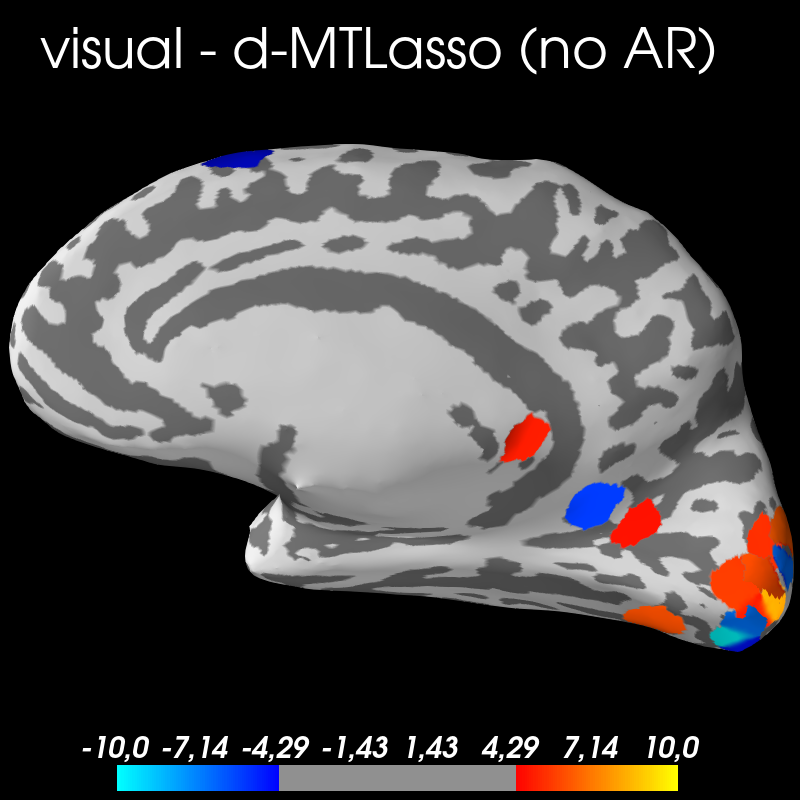}
  \includegraphics[width=0.24\textwidth]{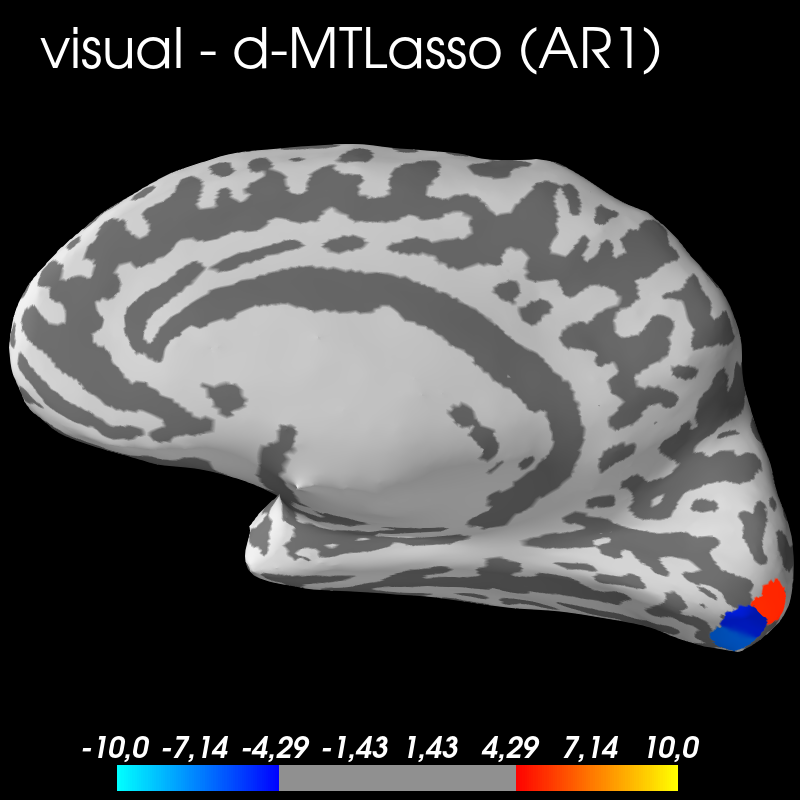}
  \includegraphics[width=0.24\textwidth]{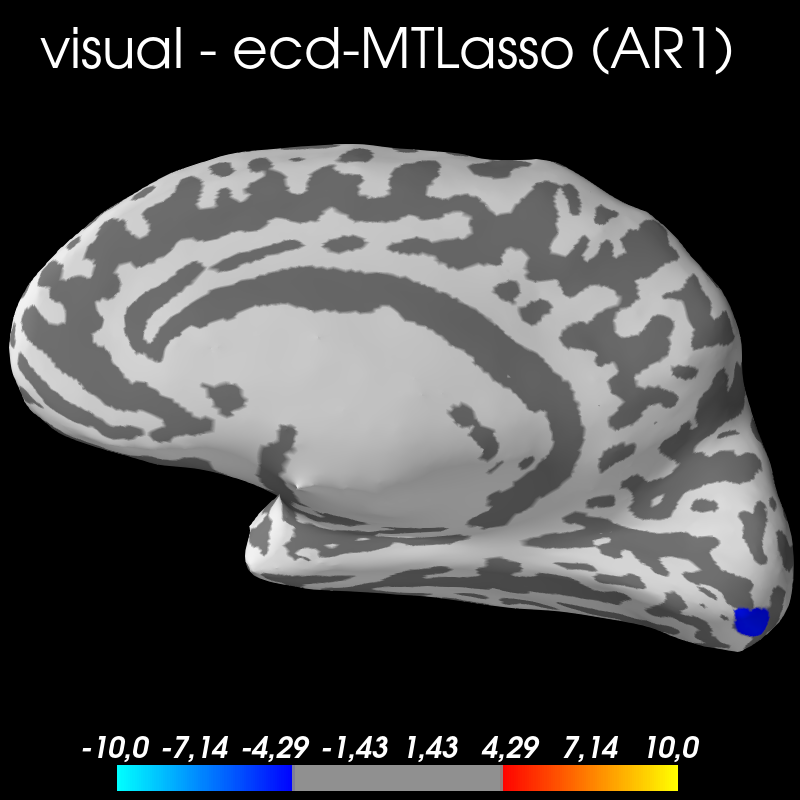}
  \includegraphics[width=0.24\textwidth]{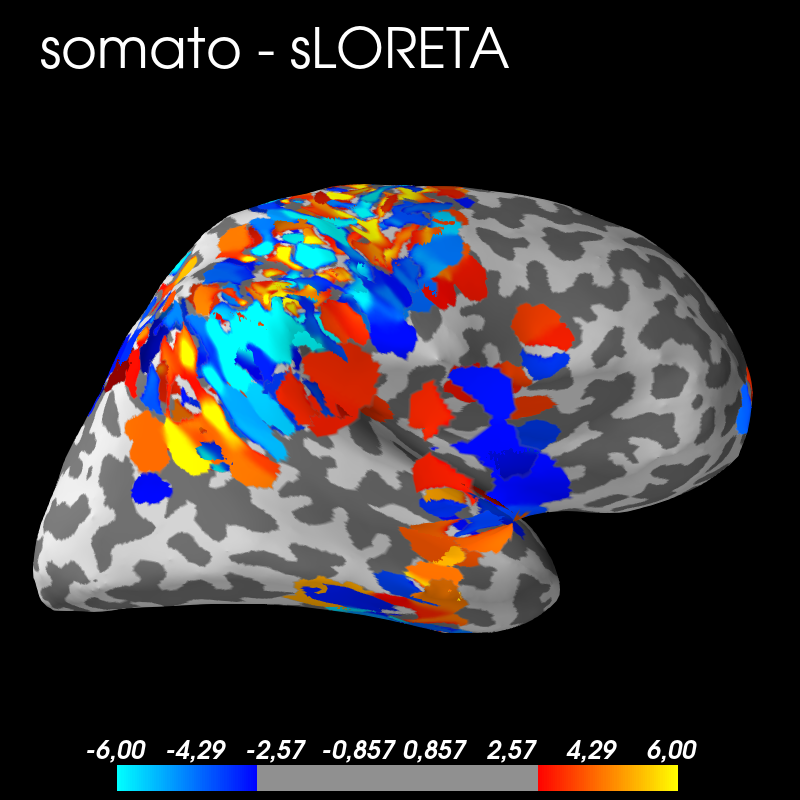}
  \includegraphics[width=0.24\textwidth]{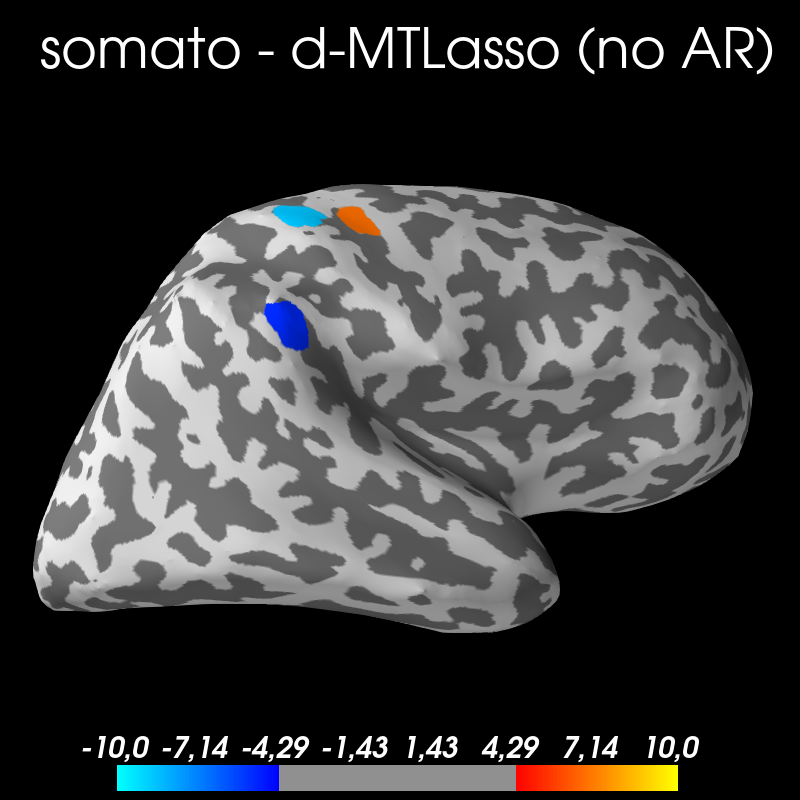}
  \includegraphics[width=0.24\textwidth]{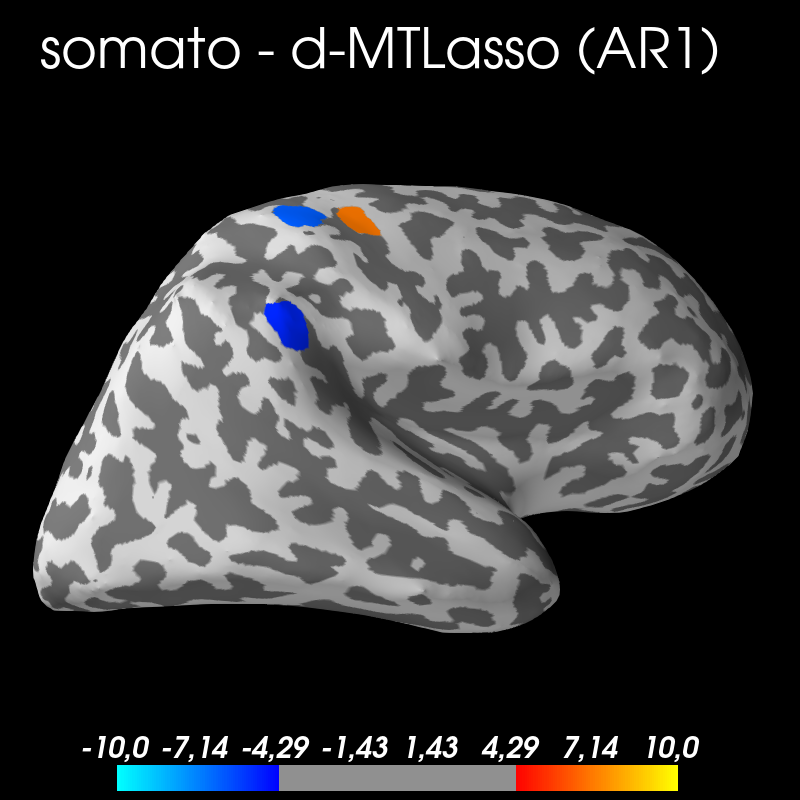}
  \includegraphics[width=0.24\textwidth]{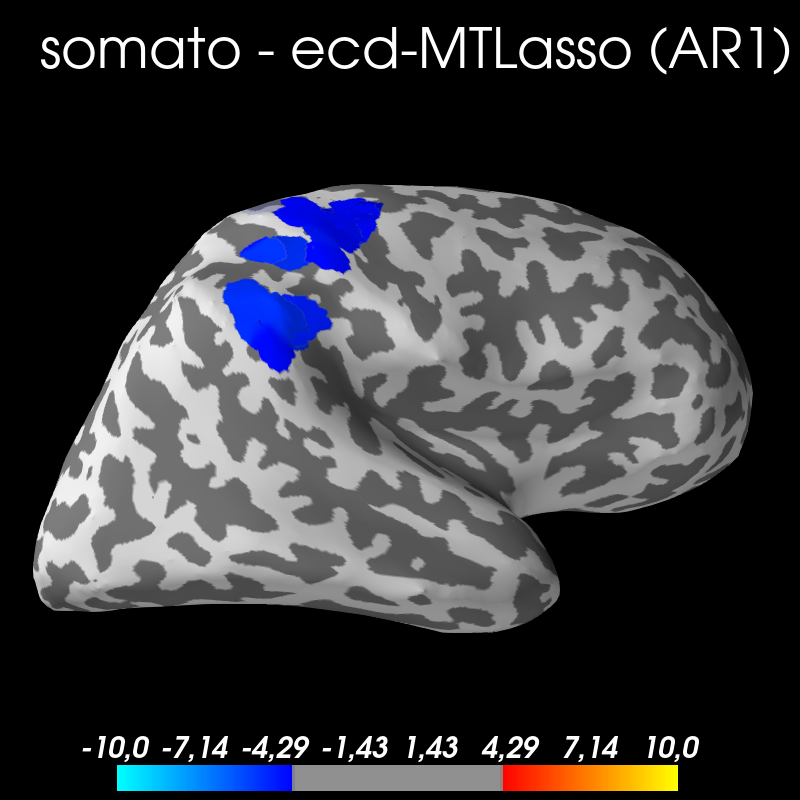}
  \caption{\textbf{Empirical comparison on 3 MEG datasets.}
  From left to right one can see sLORETA, d-MTLasso without AR modeling
  (assuming non-autocorrelated noise), d-MTLasso with an AR1 noise model
  and the ecd-MTLasso using also an AR1.
  Results correspond to auditory (top), visual (middle) and somatosensory (bottom) evoked fields.
  Colormaps are fixed across datasets and adjusted based on
  meaningful statistical thresholds in order to qualitatively
  illustrate FWER control issues.  }
  \label{fig:real_data_comparison}
\end{figure}

\subsection{Summary, guidelines and limitations}
\label{subsec:summary_exp_meg}

\paragraph{\textbf{Summary of experiments.}}
In \Cref{subsec:sparse_simu}, we have shown that taking
into account the time dimension improve the results in terms of PLE.
Also, we have seen that even in this adversarial
point source scenario (cf. \Cref{subsec:sparse_simu}),
clustered methods remain competitive.
In \Cref{subsec:semi_real_simu}, while no control of false discoveries
is proposed by sLORETA, ecd-MTL is the only method that offers
statistical control in practice.
Namely, it controls the $\delta$-FWER for $\delta$ equals to twice the
average cluster diameter.
Additionally, in this realistic simulation, ecd-MTL exhibits the
best support recovery properties.
In \Cref{subsec:real_meg}, working on real MEG data, we show that,
contrary to sLORETA, ecd-MTLasso produces calibrated statistics
with universal threshold and retrieves expected patterns without
making dubious discoveries.
Overall, ecd-MTL offers statistical guarantees and is our privileged
method.

\paragraph{\textbf{Guidelines for statistical inference with ecd-MTLasso
on temporal M/EEG data.}}
First, we try to give guidelines concerning the number of clusters $C$.
\citet{hoyos2015improving} exhibit that clustering improves
problem conditioning, this means that the Restricted Eigenvalue (RE)
property (see assumptions A1 and A3) is more likely to be verified.
Complementary, we argue that, keeping $C$ over a hundred
(limiting compression loss), the fewer clusters,
the more A3 is likely to be verified for \Cref{prop:cdMTL} and
\Cref{prop:ecdMTL} to hold but also the better the sensitivity of ecd-MTL.
However, small $C$ also requires a higher spatial tolerance.
We then hit a fundamental trade-off for statistical inference
between sensitivity and spatial specificity.
Then, $C$ can be chosen depending on the problem setting:
if it is difficult (noisy), it seems natural to lower
spatial tolerance expectations (diminish $C$); in that sense
ecd-MTL is an adaptive method (cf. \Cref{fig:delta-FWER table}).
For the present use case, taking $C = 1000$ seems an adequate
trade-off to ensure $\delta$-FWER control with reasonable spatial
tolerance.

Now, we give recommendation for time sampling and window size.
Choosing too short windows complicate AR model estimation
due to the lack of data, while choosing too large windows
may lead to non stationary support.
We recommend taking windows of $20$ to $50$ms with a
time sampling at $5$ to $10$ms as keeping $T < 10$ reduces
computation time and should not decrease sensitivity significantly.

Finally, when working with M/EEG data, we recommend to use only
$10\%$ of the full data to compute several clustering solutions
with spatial constraint and Ward criteria to ensure enough diversity.

\paragraph{\textbf{Limitations.}}
The main limitation is the fact that mixing different types
of sensors violates modeling assumptions both on temporal
correlations and on spatial correlations, that is why we had to
treat MEG and EEG sensors separately.
A possibility to handle heterogeneous sensors is to follow
\citet{massias2018generalized}, but for the temporal part
further developments are required and left for future work.

Also left for future work, is the possibility of studying windows
larger than $50$ms. A simple solution is to slide a window of
$20$ to $50$ms over the considered period of time.

Finally, a more common limitation is the fact that assumptions
are hard to test in practice.

\section{Conclusion}
\label{sec:conclusion}

The MEG source imaging problem poses a hard statistical inference challenge: namely that of high-dimensional statistical analysis, furthermore with high correlations in the design.
We have proposed an estimator that calibrates correctly the effects size and variance, up to a number of hypotheses, that are not easily met: some level of sparsity, mild correlation across sensors, homogeneity and heteroscedasticity of the noise.
Up to these hypotheses, and up to a spatial tolerance on the exact location of the sources, we provide the first method with statistical guarantees for source imaging.
This is made possible by bringing several improvements to the original
desparsified Lasso solution: a multi-task formulation that increases
power by basing inference on multiple time steps, a clustering step
that renders the design less ill-posed and an ensembling step that
mitigates the (hard) choice of clusters.
Finally, our privileged method, \ecdMTLasso, runs in less than 10\,mn on a
real dataset on non-specialized hardware, making it usable by practitioners.

\newpage

\section{Statement of broader impact}
\label{sec:impact}


Magnetoencephalography (MEG) and electroencephalography (EEG) offer a
unique opportunity to image brain activity non-invasively with a
temporal resolution in the order of milliseconds.
This is relevant for cognitive neuroscience
to describe the sequence of active areas during certain cognitive
tasks, but also for clinical neuroscience, where electrophysiology
is used for diagnosis (\eg sleep medicine, epilepsy presurgical mapping).
Yet, doing brain imaging with M/EEG requires to solve a
challenging high-dimensional inverse problem for which statistical
guarantees are crucially important.
In this work, we address this statistical challenge when using
sparsity promoting regularization and when considering the
specificity of M/EEG signals: data are spatio-temporal and the noise
is temporally autocorrelated.
The proposed algorithm is built on very recent work in optimization to
speed up Lasso-type solvers, as well as work in mathematical
statistics on desparsified Lasso estimators.
We believe that this work, whose contribution is both on the modeling side
and on the inference aspects, brings sparse estimators close to
a wide adoptions in the neuroscience community.

We also would like to emphasize that the inference framework can be
adapted to many other high-dimensional problems where data
structure can be leveraged: biomedical data and physical observations
(cardiac or brain monitoring, genomics, seismology, etc.), especially those that involve severely ill-posed inverse problems.

\paragraph{Acknowledgements}
This research is supported under funding of French ANR project FastBig
(ANR-17- CE23-0011), the KARAIB AI chair (ANR-20-CHIA-0025-01), the
European Research Council Starting Grant SLAB ERC-StG-676943 and Labex
DigiCosme (ANR-11-LABEX-0045- DIGICOSME).

\bibliographystyle{plainnat}
\bibliography{biblio}

\providecommand{\AC}{{A.-C}}\providecommand{\AM}{{A.-M}}\providecommand{\CA}{{C.-A}}\providecommand{\CH}{{C.-H}}\providecommand{\CC}{{C.-C}}\providecommand{\CJ}{{C.-J}}\providecommand{\CN}{{C.-N}}\providecommand{\DY}{{D.-Y}}\providecommand{\HJ}{{H.-J}}\providecommand{\HT}{{H.-T}}\providecommand{\HY}{{H.-Y}}\providecommand{\JA}{{J.-A}}\providecommand{\JB}{{J.-B}}\providecommand{\JC}{{J.-C}}\providecommand{\JF}{{J.-F}}\providecommand{\JJ}{{J.-J}}\providecommand{\JL}{{J.-L}}\providecommand{\JM}{{J.-M}}\providecommand{\JP}{{J.-P}}\providecommand{\JS}{{J.-S}}\providecommand{\JY}{{J.-Y}}\providecommand{\KC}{{K.-C}}\providecommand{\KR}{{K.-R}}\providecommand{\KW}{{K.-W}}\providecommand{\KL}{{K.-L}}\providecommand{\LJ}{{L.-J}}\providecommand{\MR}{{M.-R}}\providecommand{\PL}{{P.-L}}\providecommand{\RE}{{R.-E}}\providecommand{\SJ}{{S.-J}}\providecommand{\TB}{{T.-B}}\providecommand{\XR}{{X.-R}}\providecommand{\WX}{{W.-X}}\providecommand{\YX}{{Y.-X}}
\begin{thebibliography}{51}
\providecommand{\natexlab}[1]{#1}
\providecommand{\url}[1]{\texttt{#1}}
\expandafter\ifx\csname urlstyle\endcsname\relax
  \providecommand{\doi}[1]{doi: #1}\else
  \providecommand{\doi}{doi: \begingroup \urlstyle{rm}\Url}\fi

\bibitem[Baillet et~al.(2001)Baillet, Mosher, and Leahy]{baillet-etal:2001}
S.~Baillet, J.~C. Mosher, and R.~M. Leahy.
\newblock Electromagnetic brain mapping.
\newblock \emph{{IEEE} Signal Proc. Mag.}, 18\penalty0 (6):\penalty0 14--30,
  Nov. 2001.

\bibitem[Barber and Cand{\`e}s(2019)]{Barber:2016}
R.~F. Barber and E.~J. Cand{\`e}s.
\newblock A knockoff filter for high-dimensional selective inference.
\newblock \emph{Ann. Statist.}, 47\penalty0 (5):\penalty0 2504--2537, 2019.

\bibitem[Benjamini and Hochberg(1995)]{Benjamini_Hochberg95}
Y.~Benjamini and Y.~Hochberg.
\newblock Controlling the {False} {Discovery} {Rate}: {A} {Practical} and
  {Powerful} {Approach} to {Multiple} {Testing}.
\newblock \emph{J. R. Stat. Soc. Ser. B Stat. Methodol.}, 57\penalty0
  (1):\penalty0 289--300, 1995.

\bibitem[B{\"u}hlmann(2013)]{Buhlmann13}
P.~B{\"u}hlmann.
\newblock Statistical significance in high-dimensional linear models.
\newblock \emph{Bernoulli}, 19\penalty0 (4):\penalty0 1212--1242, 09 2013.

\bibitem[B{\"u}hlmann and {van de Geer}(2011)]{buhlmann2011statistics}
P.~B{\"u}hlmann and S.~{van de Geer}.
\newblock \emph{Statistics for high-dimensional data: methods, theory and
  applications}.
\newblock Springer Science \& Business Media, 2011.

\bibitem[B{\"u}hlmann et~al.(2013)B{\"u}hlmann, R{\"u}timann, {van de Geer},
  and Zhang]{Buhlmann:2013}
P.~B{\"u}hlmann, P.~R{\"u}timann, S.~{van de Geer}, and \CH. Zhang.
\newblock Correlated variables in regression: Clustering and sparse estimation.
\newblock \emph{Journal of Statistical Planning and Inference}, 143\penalty0
  (11):\penalty0 1835–1858, Nov 2013.

\bibitem[Chen and Donoho(1994)]{chen1994basis}
S.~Chen and D.~L. Donoho.
\newblock Basis pursuit.
\newblock In \emph{Proceedings of 1994 28th Asilomar Conference on Signals,
  Systems and Computers}, volume~1, pages 41--44. IEEE, 1994.

\bibitem[Chevalier et~al.(2018)Chevalier, Salmon, and
  Thirion]{Chevalier_Salmon_Thirion18}
\JA. Chevalier, J.~Salmon, and B.~Thirion.
\newblock Statistical inference with ensemble of clustered desparsified lasso.
\newblock In \emph{{International Conference on Medical Image Computing and
  Computer-Assisted Intervention}}, pages 638--646, 2018.

\bibitem[Dale et~al.(2000)Dale, Liu, Fischl, Buckner, Belliveau, Lewine, and
  Halgren]{dspm}
A.~M. Dale, A.~K. Liu, B.~R. Fischl, R.~L. Buckner, J.~W. Belliveau, J.~D.
  Lewine, and E.~Halgren.
\newblock Dynamic statistical parametric mapping.
\newblock \emph{Neuron}, 26\penalty0 (1):\penalty0 55--67, 2000.

\bibitem[Dezeure et~al.(2015)Dezeure, B{\"u}hlmann, Meier, and
  Meinshausen]{Dezeure_Buhlmann_Meier_Meinshausen15}
R.~Dezeure, P.~B{\"u}hlmann, L.~Meier, and N.~Meinshausen.
\newblock High-dimensional inference: Confidence intervals, $p$-values and
  r-software hdi.
\newblock \emph{Statist. Sci.}, 30\penalty0 (4):\penalty0 533--558, 2015.

\bibitem[Dunn(1961)]{dunn1961}
O.~J. Dunn.
\newblock Multiple comparisons among means.
\newblock \emph{J. Amer. Statist. Assoc.}, 56\penalty0 (293):\penalty0 52--64,
  1961.

\bibitem[Gimenez and Zou(2019)]{Gimenez_Zou19}
J.~R. Gimenez and J.~Zou.
\newblock Discovering conditionally salient features with statistical
  guarantees.
\newblock In \emph{ICML}, pages 2290--2298, 2019.

\bibitem[Gramfort et~al.(2012)Gramfort, Kowalski, and
  H{\"a}m{\"a}l{\"a}inen]{gramfort-etal:2012}
A.~Gramfort, M.~Kowalski, and M.~H{\"a}m{\"a}l{\"a}inen.
\newblock Mixed-norm estimates for the {M/EEG} inverse problem using
  accelerated gradient methods.
\newblock \emph{Phys. Med. Biol.}, 57\penalty0 (7):\penalty0 1937--1961, 2012.

\bibitem[Gramfort et~al.(2014)Gramfort, Luessi, Larson, Engemann, Strohmeier,
  Brodbeck, Parkkonen, and H{\"a}m{\"a}l{\"a}inen]{mne}
A.~Gramfort, M.~Luessi, E.~Larson, D.~A. Engemann, D.~Strohmeier, C.~Brodbeck,
  L.~Parkkonen, and M.~S. H{\"a}m{\"a}l{\"a}inen.
\newblock {MNE} software for processing {MEG} and {EEG} data.
\newblock \emph{{NeuroImage}}, 86:\penalty0 446--460, 2014.

\bibitem[H{\"a}m{\"a}l{\"a}inen and Ilmoniemi(1994)]{Hamalainen1994}
M.~S. H{\"a}m{\"a}l{\"a}inen and R.~J. Ilmoniemi.
\newblock Interpreting magnetic fields of the brain: minimum norm estimates.
\newblock \emph{Medical {\&} Biological Engineering {\&} Computing},
  32\penalty0 (1):\penalty0 35--42, Jan 1994.

\bibitem[Haufe et~al.(2009)Haufe, Nikulin, Ziehe, M\"{u}ller, and
  Nolte]{Haufe_Nikulin_Ziehe_Mueller_Nolte09}
S.~Haufe, V.~V. Nikulin, A.~Ziehe, \KR. M\"{u}ller, and Guido Nolte.
\newblock Estimating vector fields using sparse basis field expansions.
\newblock In D.~Koller, D.~Schuurmans, Y.~Bengio, and L.~Bottou, editors,
  \emph{NeurIPS}, pages 617--624. Curran Associates, Inc., 2009.

\bibitem[Hauk et~al.(2011)Hauk, Wakeman, and Henson]{HAUK20111966}
O.~Hauk, D.~G. Wakeman, and R.~Henson.
\newblock Comparison of noise-normalized minimum norm estimates for meg
  analysis using multiple resolution metrics.
\newblock \emph{{NeuroImage}}, 54\penalty0 (3):\penalty0 1966 -- 1974, 2011.

\bibitem[Hochberg and Tamhane(1987)]{Hochberg_Tamhane87}
Y.~Hochberg and A.~C. Tamhane.
\newblock \emph{Multiple comparison procedures}.
\newblock Wiley Series in Probability and Statistics. John Wiley \& Sons, Inc.,
  1987.

\bibitem[Hoerl and Kennard(1970)]{Hoerl_Kennard70}
A.~E. Hoerl and R.~W. Kennard.
\newblock Ridge regression: Biased estimation for nonorthogonal problems.
\newblock \emph{Technometrics}, 12\penalty0 (1):\penalty0 55--67, 1970.

\bibitem[Hoyos-Idrobo et~al.(2015)Hoyos-Idrobo, Schwartz, Varoquaux, and
  Thirion]{hoyos2015improving}
Andr{\'e}s Hoyos-Idrobo, Yannick Schwartz, Ga{\"e}l Varoquaux, and Bertrand
  Thirion.
\newblock Improving sparse recovery on structured images with bagged
  clustering.
\newblock In \emph{2015 International Workshop on Pattern Recognition in
  NeuroImaging}, pages 73--76. IEEE, 2015.

\bibitem[Janson and Su(2016)]{janson:2015}
L.~Janson and W.~Su.
\newblock Familywise error rate control via knockoffs.
\newblock \emph{Electron. J. Stat.}, 10\penalty0 (1):\penalty0 960--975, 2016.

\bibitem[Javanmard and Montanari(2014)]{Javanmard_Montanari14}
A.~Javanmard and A.~Montanari.
\newblock Confidence intervals and hypothesis testing for high-dimensional
  regression.
\newblock \emph{J. Mach. Learn. Res.}, 15:\penalty0 2869--2909, 2014.

\bibitem[Khan et~al.(2018)Khan, Hashmi, Mamashli, Michmizos, Kitzbichler,
  Bharadwaj, Bekhti, Ganesan, Garel, Whitfield-Gabrieli, Gollub, Kong, Vaina,
  Rana, Stufflebeam, H{\"a}m{\"a}l{\"a}inen, and Kenet]{KHAN201857}
S.~Khan, J.~A. Hashmi, F.~Mamashli, K.~Michmizos, M.~G. Kitzbichler,
  H.~Bharadwaj, Y.~Bekhti, S.~Ganesan, \KL.~A. Garel, S.~Whitfield-Gabrieli,
  R.~L. Gollub, J.~Kong, L.~M. Vaina, K.~D. Rana, S.~M. Stufflebeam, M.~S.
  H{\"a}m{\"a}l{\"a}inen, and T.~Kenet.
\newblock Maturation trajectories of cortical resting-state networks depend on
  the mediating frequency band.
\newblock \emph{NeuroImage}, 174:\penalty0 57 -- 68, 2018.

\bibitem[Lin et~al.(2006)Lin, Witzel, Ahlfors, Stufflebeam, Belliveau, and
  H{\"a}m{\"a}l{\"a}inen]{Lin:2006}
F.~H. Lin, T.~Witzel, S.~P. Ahlfors, S.~M. Stufflebeam, J.~W. Belliveau, and
  M.~S. H{\"a}m{\"a}l{\"a}inen.
\newblock Assessing and improving the spatial accuracy in meg source
  localization by depth-weighted minimum-norm estimates.
\newblock \emph{NeuroImage}, 31\penalty0 (1):\penalty0 160--71, 2006.

\bibitem[Lounici et~al.(2011)Lounici, Pontil, {van de Geer}, and
  Tsybakov]{Lounici_Pontil_vandeGeer_Tsybakov11}
K.~Lounici, M.~Pontil, S.~{van de Geer}, and A.~B. Tsybakov.
\newblock Oracle inequalities and optimal inference under group sparsity.
\newblock \emph{Ann. Statist.}, 39\penalty0 (4):\penalty0 2164--2204, 2011.

\bibitem[Lucka et~al.(2012)Lucka, Pursiainen, Burger, and
  Wolters]{Lucka-etal:2012}
F.~Lucka, S.~Pursiainen, M.~Burger, and C.~Wolters.
\newblock Hierarchical bayesian inference for the {EEG} inverse problem using
  realistic {FE} head models: Depth localization and source separation for
  focal primary currents.
\newblock \emph{NeuroImage}, 61\penalty0 (4):\penalty0 1364--1382, Apr. 2012.

\bibitem[Mandozzi and B{\"u}hlmann(2016)]{Mandozzi:2016}
J.~Mandozzi and P.~B{\"u}hlmann.
\newblock Hierarchical testing in the high-dimensional setting with correlated
  variables.
\newblock \emph{J. Amer. Statist. Assoc.}, 111\penalty0 (513):\penalty0
  331--343, 2016.

\bibitem[Massias et~al.(2018{\natexlab{a}})Massias, Gramfort, and
  Salmon]{Massias_Gramfort_Salmon18}
M.~Massias, A.~Gramfort, and J.~Salmon.
\newblock {Celer: a Fast Solver for the Lasso with Dual Extrapolation}.
\newblock In \emph{ICML}, volume~80, pages 3315--3324, 2018{\natexlab{a}}.

\bibitem[Massias et~al.(2019)Massias, Vaiter, Gramfort, and
  Salmon]{Massias_Vaiter_Gramfort_Salmon19}
M.~Massias, S.~Vaiter, A.~Gramfort, and J.~Salmon.
\newblock Dual extrapolation for sparse generalized linear models.
\newblock \emph{arXiv preprint arXiv:1907.05830}, 2019.

\bibitem[Massias et~al.(2018{\natexlab{b}})Massias, Fercoq, Gramfort, and
  Salmon]{massias2018generalized}
Mathurin Massias, Olivier Fercoq, Alexandre Gramfort, and Joseph Salmon.
\newblock Generalized concomitant multi-task lasso for sparse multimodal
  regression.
\newblock In \emph{International Conference on Artificial Intelligence and
  Statistics}, pages 998--1007, 2018{\natexlab{b}}.

\bibitem[Matsuura and Okabe(1995)]{Matsuura-Okabe:1995}
K.~Matsuura and Y.~Okabe.
\newblock Selective minimum-norm solution of the biomagnetic inverse problem.
\newblock \emph{{IEEE} Trans. Biomed. Eng.}, 42\penalty0 (6):\penalty0
  608--615, June 1995.
\newblock ISSN 0018-9294.

\bibitem[Mattout et~al.(2005)Mattout, P{\'e}l{\'e}grini-Issac, Garnero, and
  Benali]{MATTOUT2005356}
J.~Mattout, M.~P{\'e}l{\'e}grini-Issac, L.~Garnero, and H.~Benali.
\newblock Multivariate source prelocalization ({MSP}): Use of functionally
  informed basis functions for better conditioning the {MEG} inverse problem.
\newblock \emph{NeuroImage}, 26\penalty0 (2):\penalty0 356 -- 373, 2005.

\bibitem[Meinshausen and B{\"u}hlmann(2006)]{Meinshausen_Buhlmann06}
N.~Meinshausen and P.~B{\"u}hlmann.
\newblock High-dimensional graphs and variable selection with the lasso.
\newblock \emph{Ann. Statist.}, 34\penalty0 (3):\penalty0 1436--1462, 2006.

\bibitem[Meinshausen and B{\"u}hlmann(2010)]{Meinshausen2010}
N.~Meinshausen and P.~B{\"u}hlmann.
\newblock Stability selection.
\newblock \emph{J. R. Stat. Soc. Ser. B Stat. Methodol.}, 72:\penalty0
  417--473, 2010.

\bibitem[Meinshausen et~al.(2009)Meinshausen, Meier, and
  B{\"u}hlmann]{Meinshausen_Meier_Buhlmann09}
N.~Meinshausen, L.~Meier, and P.~B{\"u}hlmann.
\newblock P-values for high-dimensional regression.
\newblock \emph{J. Amer. Statist. Assoc.}, 104\penalty0 (488):\penalty0
  1671--1681, 2009.

\bibitem[Mitra and Zhang(2016)]{Mitra_Zhang16}
R.~Mitra and \CH. Zhang.
\newblock The benefit of group sparsity in group inference with de-biased
  scaled group lasso.
\newblock \emph{Electron. J. Stat.}, 10\penalty0 (2):\penalty0 1829--1873,
  2016.

\bibitem[Molins et~al.(2008)Molins, Stufflebeam, Brown, and
  H{\"a}m{\"a}l{\"a}inen]{MOLINS20081069}
A.~Molins, S.~M. Stufflebeam, E.~N. Brown, and M.~S. H{\"a}m{\"a}l{\"a}inen.
\newblock Quantification of the benefit from integrating {MEG} and {EEG} data
  in minimum {L2}-norm estimation.
\newblock \emph{NeuroImage}, 42\penalty0 (3):\penalty0 1069 -- 1077, 2008.

\bibitem[Nguyen et~al.(2019)Nguyen, Chevalier, and
  Thirion]{Nguyen_Chevalier_Thirion19}
\TB. Nguyen, \JA. Chevalier, and B.~Thirion.
\newblock Ecko: Ensemble of clustered knockoffs for robust multivariate
  inference on {fMRI} data.
\newblock In \emph{International Conference on Information Processing in
  Medical Imaging}, pages 454--466, 2019.

\bibitem[Obozinski et~al.(2010)Obozinski, Taskar, and
  Jordan]{Obozinski_Taskar_Jordan10}
G.~Obozinski, B.~Taskar, and M.~I. Jordan.
\newblock Joint covariate selection and joint subspace selection for multiple
  classification problems.
\newblock \emph{Statistics and Computing}, 20\penalty0 (2):\penalty0 231--252,
  2010.

\bibitem[Ou et~al.(2009)Ou, H{\"a}mal{\"a}inen, and
  Golland]{Ou_Hamalainen_Golland09}
W.~Ou, M.~S. H{\"a}mal{\"a}inen, and P.~Golland.
\newblock A distributed spatio-temporal {EEG}/{MEG} inverse solver.
\newblock \emph{NeuroImage}, 44\penalty0 (3):\penalty0 932--946, Feb. 2009.

\bibitem[Pascual-Marqui(2002)]{PascualMarqui:2002}
R.~Pascual-Marqui.
\newblock Standardized low resolution brain electromagnetic tomography
  ({sLORETA}): technical details.
\newblock \emph{Methods Find. Exp. Clin. Pharmacology}, 24\penalty0
  (D):\penalty0 5--12, 2002.

\bibitem[Reid et~al.(2016)Reid, Tibshirani, and
  Friedman]{Reid_Tibshirani_Friedman16}
S.~Reid, R.~Tibshirani, and J.~Friedman.
\newblock A study of error variance estimation in lasso regression.
\newblock \emph{Stat. Sin.}, 26\penalty0 (1):\penalty0 35--67, 2016.

\bibitem[Stucky and {van de Geer}(2018)]{Stucky_vandeGeer18}
B.~Stucky and S.~{van de Geer}.
\newblock Asymptotic confidence regions for high-dimensional structured
  sparsity.
\newblock \emph{{IEEE} Trans. Signal Process.}, 66\penalty0 (8):\penalty0
  2178--2190, 2018.

\bibitem[Taulu(2006)]{taulu:06}
S.~Taulu.
\newblock Spatiotemporal {S}ignal {S}pace {S}eparation method for rejecting
  nearby interference in {MEG} measurements.
\newblock \emph{Physics in Medicine and Biology}, 51\penalty0 (7):\penalty0
  1759--1769, 2006.

\bibitem[Taylor and Tibshirani(2015)]{Taylor:2015}
J.~Taylor and R.~J. Tibshirani.
\newblock Statistical learning and selective inference.
\newblock \emph{Proceedings of the National Academy of Sciences}, 112\penalty0
  (25):\penalty0 7629--7634, 2015.

\bibitem[Tibshirani(1996)]{tibshirani1996}
R.~Tibshirani.
\newblock Regression shrinkage and selection via the lasso.
\newblock \emph{J. R. Stat. Soc. Ser. B Stat. Methodol.}, 58\penalty0
  (1):\penalty0 267--288, 1996.

\bibitem[{van de Geer} et~al.(2014){van de Geer}, B{\"u}hlmann, Ritov, and
  Dezeure]{vandeGeer_Buhlmann_Ritov_Dezeure14}
S.~{van de Geer}, P.~B{\"u}hlmann, Y.~Ritov, and R.~Dezeure.
\newblock On asymptotically optimal confidence regions and tests for
  high-dimensional models.
\newblock \emph{Ann. Statist.}, 42\penalty0 (3):\penalty0 1166--1202, 2014.

\bibitem[Varoquaux et~al.(2012)Varoquaux, Gramfort, and
  Thirion]{Varoquaux_Gramfort_Thirion12}
G.~Varoquaux, A.~Gramfort, and B.~Thirion.
\newblock Small-sample brain mapping: sparse recovery on spatially correlated
  designs with randomization and clustering.
\newblock In \emph{ICML}, pages 1375--1382, 2012.

\bibitem[Wasserman and Roeder(2009)]{Wasserman_Roeder09}
L.~Wasserman and K.~Roeder.
\newblock High-dimensional variable selection.
\newblock \emph{Ann. Statist.}, 37\penalty0 (5A):\penalty0 2178--2201, 2009.

\bibitem[Wipf and Nagarajan(2009)]{Wipf-Nagarajan:2009}
D.~Wipf and S.~Nagarajan.
\newblock A unified bayesian framework for {MEG/EEG} source imaging.
\newblock \emph{NeuroImage}, 44\penalty0 (3):\penalty0 947--966, Feb. 2009.

\bibitem[Zhang and Zhang(2014)]{Zhang_Zhang14}
\CH. Zhang and S.~S. Zhang.
\newblock Confidence intervals for low dimensional parameters in high
  dimensional linear models.
\newblock \emph{J. R. Stat. Soc. Ser. B Stat. Methodol.}, 76\penalty0
  (1):\penalty0 217--242, 2014.

\end{thebibliography}

\clearpage

\appendix
\begin{center}
{\centering \LARGE APPENDIX}
\vspace{1cm}
\sloppy

\end{center}

\section{Formal definition of $\delta$-FWER control}
\label{sec:complement_metrics}

Now we give a more formal definition of the $\delta$-FWER.
\begin{df}[$\delta$-family wise error rate]\label{df:delta_fwer}
Given a family of (corrected) $p$-values $\hat{p} = (\hat{p}_j)_{j \in [p]}$
and a threshold $x \in (0,1)$, the $\delta$-FWER, also denoted
${\rm FWER}^{\delta}_{x}(\hat{p})$, is the probability
to make at least one false discovery at a distance at least $\delta$
from the true support:
\begin{align}
{\rm FWER}^{\delta}_{x}(\hat{p})
= \bbP(\min_{j \in N^{\delta}}\hat{p}_j \leq x )
\enspace ,
\end{align}
with $N^{\delta} \! = \left\{j \!\in \![p] \!:\! \forall k \in \mathrm{Supp}(\*B),
d(j,k) \geq \delta \right\}$ and $d(j,k)$ is the distance between source
$j$  and $k$.
\end{df}

\begin{df}[$\delta$-FWER control]\label{df:control_delta_fwer}
We say that the family of (corrected) $p$-values
$\hat{p} = (\hat{p}_j)_{j \in [p]}$
controls the $\delta$-FWER if, for all $x \in (0,1)$:
\begin{align}\label{eq:control_delta_fwer}
{\rm FWER}^{\delta}_{x}(\hat{p}) \leq x \enspace .
\end{align}
\end{df}

\section{Extended Restricted Eigenvalue assumption}
\label{sec:RE_assumption}
%
Here, we rewrite \citep[Assumption 3.1]{Lounici_Pontil_vandeGeer_Tsybakov11},
adjusting it for the multi-task Lasso case (particular case
of the more general group Lasso).
Notice that for a given value of $T$, the assumption is equivalent to
\citep[Assumption 4.1]{Lounici_Pontil_vandeGeer_Tsybakov11}.
Let $1 \leq s \leq p$ be an integer that gives an upper bound
on the sparsity $|\mathrm{Supp}(\*B)|$.
The extended Restricted Eigenvalue assumption RE($\*X,s$) is verified
on $\*X$ for sparsity parameter $s$ and constant $\kappa = \kappa(s) > 0$,
if:
\begin{align}\label{eq:RE_assumption}
\min \left\{
\frac{\norm{\*X \bm\Theta}}{\sqrt{nT}\norm{\bm\Theta_{J}}} :
|J|\leq s, \bm\Theta \in \bbR^{p \times T} \setminus \{ \*0 \},
\norm{\bm\Theta_{J^{C}}}_{2,1} \leq 3 \norm{\bm\Theta}_{2,1}
\right\} \geq \kappa \enspace ,
\end{align}
where $J \subset [p]$ and $J^{C}$ denotes its complementary \ie
$J^{C} = [p] \setminus J$, and $\bm\Theta_{J}$ refers to the matrix
$\bm\Theta$ without the rows $J^{C}$.
%

\section{Adaptive quantile aggregation of $p$-values and ecd-MTLasso algorithm}
\label{sec:complement}

In this section, we provide some more details on the way we perform
aggregation of $p$-values across the $p$-values maps created
through the clustering randomization, then we give the full
ecd-MTLasso algorithm.

For the $j$-th features (or source) we have a vector $(p_j^{(b)})_{b \in [B]}$ of $p$-values, with one $p$-value computed for each of the $B$ clusterings.
Then, the final $p$-value of the $j$-th feature is given by the
adaptive quantile aggregation, as proposed by
\citet{Meinshausen_Meier_Buhlmann09}:
\begin{align*}\label{eq:pval_aggreg}
	p_j =
	\min \left\{ (1-\log(\gamma_{\min}))
	\inf_{\gamma \in (\gamma_{\min}, 1)}
	\Big(
		\gamma \mbox{-quantile}
		\Big\{ \frac{p^{(b)}_j}{\gamma}; b \in [B] \Big\}
	\Big)
	~, ~~ 1
	\right\}
	\enspace ,
\end{align*}
where we have taken $\gamma_{\min} = 0.25$ in our experiments. Taking
a value of $\gamma_{\min}$ not too small (\eg $\gamma_{\min} \geq 0.25$)
allows to recover sources that have received small $p$-values
several times (\eg at least for $B / 4$ different choices of clustering).

We give the full algorithm of ecd-MTLasso in \Cref{alg:ecdMTLasso}.

{\fontsize{4}{4}\selectfont
\begin{algorithm}[ht]
\SetKwInOut{Input}{input}
\SetKwInOut{Output}{output}
\SetKwInOut{Parameter}{param}
\caption{ecd-MTLasso}
\Input{$\*X \in \bbR^{n \times p}, \*Y$}

\BlankLine

\Parameter{$C=1000, B=100$\\}

\BlankLine
\For{$b = 1, \dots, B$}
{\vspace {2mm}
  $ ~~ \*X^{(b)} = \texttt{sample}(\*X)$ \\
  $ \*A^{(b)} = \texttt{Ward}(C, \*X^{(b)}) $ \\
  $ \*Z^{(b)} = \*X \*A^{(b)}$ \\
  $ q^{(b)} = \frac{\texttt{d-MTLasso}(\*Z^{(b)}, \*Y)}{C}$
  \tcp*{corrected cluster-wise $p$-values at bootstrap $b$}
  \vspace {2mm}
  \For{$j = 1, \dots, p $}
  {$ p_j^{(b)} = q_r^{(b)}$ if $j \in G_r $
  \tcp*{corrected feature-wise $p$-values at bootstrap $b$}}
}
\BlankLine
\For{$j = 1, \dots, p $}
{$ p_j = \texttt{aggregation}(p_j^{(b)}, b \in [B])$
\tcp*{aggregated corrected feature-wise $p$-values}}
\BlankLine
\Return $p_j$ for $j \in [p]$
\label{alg:ecdMTLasso}
\end{algorithm}
}

\section{Proofs}
\label{sec:proof}

\subsection{Probability lemma}
\label{sub:probability_lemma}

\begin{lem}
	Let $\bm\varepsilon \in \bbR^{T}$ be a centered Gaussian random vector with (symmetric positive definite) covariance $\*M \in \bbR^{T\times T}$.
	Then, the random variable $\bm\varepsilon^\top \*M^{-1} \bm\varepsilon$ follows a $\chi^2_T$ distribution.
\end{lem}
\begin{proof}
	Note first that since $\*M$ is symmetric positive definite, its square-root $\*N \in \bbR^{T\times T}$ exists and is a symmetric positive definite matrix satisfying $\*N^2 = \*M$.
	Hence, this leads to the following displays
	\begin{align*}
		\bm\varepsilon^\top \*M^{-1} \bm\varepsilon = (\*N^{-1} \bm\varepsilon)^\top (\*N^{-1} \bm\varepsilon).
	\end{align*}
	We have that $\*N^{-1} \bm\varepsilon$ is a centered Gaussian random vector, and its covariance matrix reads:
	\begin{align*}
		\bbE \left[( \*N^{-1} \bm\varepsilon) (\*N^{-1} \bm\varepsilon)^\top \right]
		& = \bbE \left[ \*N^{-1} \bm\varepsilon  \bm\varepsilon^\top \*N^{-1} \right]\\
		& = \bbE \left[ \*N^{-1} \bm\varepsilon  \bm\varepsilon^\top \*N^{-1} \right]\\
		& =   \*N^{-1} \bbE\left[\bm\varepsilon  \bm\varepsilon^\top \right] \*N^{-1}\\
		& =   \*N^{-1} \*M \*N^{-1}\\
		& =   \*N^{-1} \*N^2 \*N^{-1}\\
		& =  \Id_T \enspace.
	\end{align*}
	To conclude $\*N^{-1} \bm\varepsilon \in\bbR^T$ is a centered Gaussian vector with covariance $\Id_T$, hence its squared Euclidean norm $ \norm{\*N^{-1} \bm\varepsilon}^2 = (\*N^{-1} \bm\varepsilon)^\top (\*N^{-1} \bm\varepsilon)$ follows a $\chi^2_T$ distribution.
\end{proof}

\subsection{Proof of \Cref{prop:desparsified_mtlasso}}
\label{sub:proof_of_prop:desparsified_mtlasso}

Now, we give a proof of \Cref{prop:desparsified_mtlasso}:

\begin{proof}
First, let us fix an index $j \in [p]$. Then, using \Cref{eq:debiasied_MTL} we have:
\begin{align}
  \begin{split}
	\sqrt{n}(\hat{\*B}_{j,.}^{{(\rm \dMTLasso)}} - \*B_{j,.})
	& =
  \sqrt{n} \frac{\*z_{j}^\top \*E}{\*z_{j}^\top \*X_{.,j}}
	- \sum_{k \neq j} \frac{\sqrt{n} \, \*z_{j}^\top \*X_{.,k}
   (\hat{\*B}^{\rm{MTL}}_{k,.} - \*B_{k,.})}
   {\*z_{j}^\top \*X_{.,j}}  \\
    & = \bm\Lambda_{j,.} + \bm\Delta_{j,.}
    \enspace,
  \end{split}
\end{align}
where $\bm\Lambda_{j,.} = \sqrt{n} \frac{\*z_{j}^\top \*E}{\*z_{j}^\top \*X_{.,j}}$
and $\bm\Delta_{j,.} = \sqrt{n} \sum_{k \neq j} \*P_{j,k}
(\*B_{k,.} - \hat{\*B}^{\rm{MTL}}_{k,.})$ with
$$\*P_{j,k} = \frac{\*z_{j}^\top \*X_{.,k}}
{\*z_{j}^\top \*X_{.,j}} \enspace .$$

Now, we show that $\bm\Lambda_{j, .} \sim \mathcal{N}_{p}(\*0, \, \hat{\bm\Omega}_{j, j} \*M)$, or equivalently we show that
$\*E^\top \*z_{j} \sim  \cN(0,n \norm{\*z_{j}}^2 \*M)$.
It is clear that $\*E^\top \*z_{j}$ is a centered Gaussian vector.
Then, its covariance denoted by $\*V^{(j)}$, can be computed as follows:
\begin{align*}
	\*V^{(j)} = \bbE(\*E^\top \*z_{j} \*z_{j}^\top \*E ) \in \bbR^{T\times T} \enspace,
\end{align*}
whose general term is given for $t,t' \in [T]$ by
\begin{align*}
	\*V_{t,t'}^{(j)}
	& = \bbE(\*E_{.,t}^\top \*z_{j} \*z_{j}^\top \*E_{.,t'})\\
	& = \bbE(\*z_{j}^\top \*E_{.,t'} \*E_{.,t}^\top \*z_{j}) \quad\quad\quad\quad(\text{scalar values commute})\\
	& = \*z_{j}^\top \bbE(\*E_{.,t'} \*E_{.,t}^\top) \*z_{j}\\
	& = \*z_{j}^\top \bbE(\sum_{i=1}^n\*E_{i,t'} \*E_{i,t}^\top) \*z_{j}\\
	& = \*z_{j}^\top \sum_{i=1}^n \bbE(\*E_{i,t'} \*E_{i,t}^\top) \*z_{j} \enspace.
\end{align*}
Then, the noise structure in \Cref{eq:noise_law_2} yields
$\*V_{t,t'}^{(j)} = \*z_{j}^\top n \*M_{t,t'} \*z_{j} = n \norm{\*z_j}^2\*M_{t,t'}$.

Now, we show that with high probability $\norm{\bm\Delta}_{2, 1} =
\mathrm{O} \left(\frac{s \lambda \sqrt{\log(p)}}{\kappa^2}\right)$.
First, notice that:
\begin{align*}
	\norm{\bm\Delta}_{2, 1}
	& \leq \sqrt{n} \max_{k \neq j} |\*P_{j,k}|
  \norm{\hat{\*B}^{\rm{MTL}} - \*B}_{2, 1}
 \enspace.
\end{align*}
For a convenient choice of the regularization parameters $\bm\alpha$,
using \citet[Lemma 2.1]{buhlmann2011statistics} and following the
same approach as \citet[Appendix A.1]{Dezeure_Buhlmann_Meier_Meinshausen15},
we obtain, with high probability:
\begin{align*}
	\sqrt{n} \max_{k \neq j} |\*P_{j,k}|
	=
	\mathrm{O}\left(\sqrt{\log(p)}\right) \enspace.
\end{align*}
Bounds on $\normin{\hat{\*B}^{\rm{MTL}} - \*B}_{2, 1}$ are also
available in the literature \citep{Lounici_Pontil_vandeGeer_Tsybakov11}
for $\rho = 0$ and can be extended to $\rho > 0$ similarly.
Notably, provided $\rho = 0$, assuming A1 for a sparsity parameter
$|\mathrm{Supp}(\*B^*)| \leq s$, a given constant $\kappa = \kappa(s) > 0$,
and a choice of $\lambda$ large enough in \Cref{eq:beta_mtl},
\citep[Theorem 3.1]{Lounici_Pontil_vandeGeer_Tsybakov11} gives directly
the following bound, with high probability:
\begin{align*}
\norm{\hat{\*B}^{\rm{MTL}} - \*B}_{2, 1} =
\mathrm{O}\left(\frac{s \lambda}{\kappa^2}\right)
\enspace.
\end{align*}

\end{proof}

\begin{rk}
  Following \citet{vandeGeer_Buhlmann_Ritov_Dezeure14}, to neglect $\bm\Delta$
  we need to have $\norm{\bm\Delta}_{\infty} = \mathrm{o(1)}$.
  This condition is verified if
  $s =  \mathrm{o}\Big(\frac{\kappa^2}{\lambda\sqrt{\log(p)}}\Big)$.
\end{rk}

\subsection{Proof of \Cref{prop:cdMTL}}
\label{sub:proof_of_prop_cdMTL}

Before starting the proof, let us give more precision on assumption A2,
the complete assumption is the following:

(A2) there exists $\bm\Gamma \in \bbR^{C \times T}$
such that $\bm\Gamma_{r, .} = \sum_{j \in G_r}{w_{j}\*B_{j, .}}$
with $w_{j} \geq 0$ for all $j \in [p]$,
so that the associated compression loss $\*X \*B - \*Z\bm\Gamma$
is bounded as follows:
\begin{align}
    \norm{\*X \*B - \*Z \bm\Gamma}^{2}_{2,2}
    \leq \xi \frac{T \phi^2_{\min}(\*M)}{n}
    =  \xi \frac{ T \phi^2_{\min}(\*R) \sigma^2}{n}
    \enspace ,
\end{align}
where $\xi > 0$ is an arbitrary small constant, $\phi^2_{\min}(\*M)
> 0$ is the smallest eigenvalue of $\*M$ and $\phi^2_{\min}(\*R)
> 0$ is the smallest eigenvalue of $\*R$, the temporal correlation
matrix of the noise defined by $\*R = \*M / \sigma^2$.
The hypothesis plainly means that the noise induced by design matrix
compression is small enough with respect to the model noise.

Now we give a proof of \Cref{prop:cdMTL}:

\begin{proof}

First, we derive the \dMTLasso for the compressed problem, for $r \in [C]$:

\begin{align}\label{eq:debiasied_MTL_compressed}
	\hat{\bm\Gamma}_{r,.}^{(\rm\dMTLasso)}
	=
	\frac{\*a_{r}^\top \*Y}{\*a_{r}^\top\*Z_{.,r}}
	- \sum_{l \neq r} \frac{\*a_{r}^\top \*Z_{.,l} \hat{\bm\Gamma}^{\rm{MTL}}_{r,.}}{\*a_{r}^\top \*Z_{.,r}}\enspace,
\end{align}

where $a_{r}$'s are the residuals obtained by nodewise Lasso on $\*Z$
playing the same role as the $z_{j}$'s in \Cref{eq:debiasied_MTL}.
Then, as done in \Cref{sub:proof_of_prop:desparsified_mtlasso}, we derive:
\begin{align}
  \begin{split}
	\sqrt{n}(\hat{\bm\Gamma}_{r,.}^{{(\rm \dMTLasso)}} - \bm\Gamma_{r,.})
	& =
	\sqrt{n} \frac{\*a_{r}^\top \*E}{\*a_{r}^\top \*Z_{.,r}}
	- \sum_{l \neq r} \frac{\sqrt{n} \, \*a_{r}^\top \*Z_{.,l}
   (\hat{\bm\Gamma}_{l,.}^{\rm{MTL}} - \bm\Gamma_{l,.})}
   {\*a_{r}^\top \*Z_{.,r}}
   + \frac{\sqrt{n} \, \*a_{r}^\top (\*X \*B - \*Z \bm\Gamma)}
    {\*a_{r}^\top \*Z_{.,r}} \\
    & = \bm\Lambda^{\prime}_{r,.} + \bm\Delta^{\prime}_{r,.} + \bm\Pi_{r,.}
    \enspace,
  \end{split}
\end{align}

We treat $\bm\Lambda^{\prime}$ and $\bm\Delta^{\prime}$ as in
\Cref{sub:proof_of_prop:desparsified_mtlasso}, assuming that the
hypotheses that are used to bound (hence, neglect)
$\bm\Delta^{\prime}$ are verified (notably A3).

Next, for $r \in [C]$, we want to establish that
$\frac{n \norm{\bm\Pi_{r,.}}_{\*M^{-1}}^2}{T \hat{\bm\Omega}^{\prime}_{r,r}}$
is negligible, \ie that $\bm\Pi$ has a negligible effect on all decision
statistics, where the covariance $\hat{\bm \Omega}^{\prime}$ has the
following generic diagonal term:
$$\hat{\bm\Omega}^{\prime}_{r,r} =
\frac{n \norm{\*a_{r}}^2}{|\*a_{r}^\top \*Z_{.,r}|^2} \enspace .$$

Given that
\begin{align}
  \norm{\bm\Pi_{r,.}}_{\*M^{-1}}^2 & =
  \frac{n \norm{\*a_{r}^\top (\*X \*B - \*Z \bm\Gamma)}_{\*M^{-1}}^2}
  {|\*a_{r}^\top \*Z_{.,r}|^2}\\
  & \leq n \frac{\norm{\*a_{r}^\top}^2}{|\*a_{r}^\top \*Z_{.,r}|^2}
  \frac{\norm{\*X \*B - \*Z \bm\Gamma}^2_{2,2}}{\phi^2_{\min}(\*M)}
  \enspace ,
\end{align}
where $\norm{\cdot}_{2,2}$ denotes the spectral norm.
Then, we obtain that
\begin{align}
\frac{n
  \norm{\bm\Pi_{r,.}}_{\*M^{-1}}^2}{T \hat{\bm\Omega}^{\prime}_{r,r}} & \leq
\frac{n}{T} \frac{\norm{\*X \*B - \*Z \bm\Gamma}^2_{2,2}}{\phi^2_{\min}(\*M)}
\leq \xi \enspace .
\end{align}

Then, if A2 is verified for $\xi$ small enough, we can also neglect
$\bm\Pi$ in front of $\bm\Lambda^{\prime}$.

Then, by neglecting $\bm\Pi$ and $\bm\Delta^{\prime}$, we have:
\begin{align}
\sqrt{n}(\hat{\bm\Gamma}^{{(\rm \dMTLasso)}} - \bm\Gamma) \sim
\mathcal{N}_{C}(\*0, \, \hat{\bm\Omega}^{\prime}_{r, r} \*M)
\enspace .
\end{align}

Then we can construct $p$-values that test the $r$-th null hypothesis
$H^{(r)}_0$ : ``$\bm\Gamma_{j,.}=0$'',
applying the same technique as in \Cref{sub:desparsified_multi_task_lasso}.
By correcting these $p$-values ---\eg using the Bonferroni correction
\citep{dunn1961}, we multiply by $C$ the initial $p$-values---,
 we obtain cluster-wise corrected $p$-values that control the FWER.

Since, for all $r \in [C]$, $\bm\Gamma_{r,.}$ is a
linear combination of $\*B_{j,.}$ for $j \in G_{r}$, then
$\bm\Gamma_{r,.} \neq 0$ if at least there exist $j \in G_{r}$
such that $\*B_{j,.} \neq 0$.

Then, defining the feature-wise corrected $p$-values by
the corrected $p$-values of the corresponding cluster,
and assuming that clusters are at most of size $\delta$,
such corrected $p$-values control the $\delta$-FWER.

\end{proof}

\begin{rk}

  In assumption A2, having a positive linear combination
  is not necessary, a simple linear combination is sufficient.

  However, we assumed that $\bm\Gamma_{r,.}$ was a positive
  linear combination of $\*B_{j,.}$ for $j \in G_{r}$,
  to get the following desired properties:

  "If additionally for $r \in [C]$, for all $j \in G_{r}$
  and all $k \in G_{r}$, we have $\sign(\*B_{j,.}) = \sign(\*B_{k,.})$,
  then $\sign(\bm\Gamma_{r,.}) = \sign(\*B_{j,.})$
  (zero being booth positive and negative)."

  This means that if all the features' weights in a cluster have the
  same sign, there exists a compression verifying A2 such that the
  cluster weight preserves the sign.

\end{rk}

\subsection{Proof of \Cref{prop:ecdMTL}}
\label{sub:proof_of_prop_ecdMTL}

\begin{proof}
Assuming the hypotheses of \Cref{prop:ecdMTL} and applying \Cref{prop:cdMTL},
we can, for each of the $B$ compression of the problem in
\Cref{eq:noise_model}, construct a corrected $p$-value family that
control the $\delta$-FWER.
Applying the quantile aggregate method in \Cref{eq:pval_aggreg},
we derive a corrected $p$-value family taking into account for
each compression choice.
Applying \citet[Theorem 3.2]{Meinshausen_Meier_Buhlmann09}, this
aggregated corrected $p$-value family
also controls the $\delta$-FWER.
\end{proof}

\section{Computational aspects}
\label{subsec:computational}
%
Here we give some elements about the computational aspect of the
algorithms we propose.

For solving Lasso or multi-task Lasso problems, we rely for additional
speed-up on \texttt{celer}\footnote{\url{https://github.com/mathurinm/CELER}}
\citep{Massias_Gramfort_Salmon18,Massias_Vaiter_Gramfort_Salmon19},
a solver which is much more efficient than
the standard coordinate descent (speed up by more than 10x on our
experiments).

To compute d-MTLasso, we must solve $p$ Lasso of size $(n, (p -1))$,
and 1 multi-task Lasso with cross-validation on a dataset of size $(n, p, T)$.
For $n = 200$, $p = 7500$ and $T = 10$, the algorithms can be run
on a standard laptop in around $10$ hours (using only 1 CPU).
However, the algorithm is embarrassingly parallel and requires
around $15$ minutes if run on a machine with $50$ CPUs.
To compute cd-MTLasso, we must solve $C$ Lasso of size $(n, (C -1))$.
and 1 multi-task Lasso with cross-validation on a dataset of size $(n, C, T)$.
For $n = 200$, $C=1000$ and $T = 10$, it can be run on a standard local
device in less than $1$ minute (using only 1 CPU).
Finally, to compute ecd-MTLasso, we must solve $B$ cd-MTLasso.
For $B = 100$ ($25$ is already a good value to get most of the advantages
of ensembling), $n = 200$, $C=1000$ and $T = 10$, it can be run on a
standard laptop in around $1$ hour (using only 1 CPU)
and around $1$ minute on a machine with $50$ CPUs.

Although, when using coordinate-descent-like algorithms, the complexity depends
on solver parameters such as tolerance on stopping criteria, the complexity in $C$
(or $p$) appears empirically to be cubic, while it is linear in $n$ and $T$.
It is also linear in $B$.

\section{Detailed data description}
\label{sec:data}
For AEF and VEF, data contained one artifactual channel leading to $n=203$,
while for SEF data were preprocessed for removal of environmental noise
leading to an effective number of samples of $n=64$~\citep{taulu:06}.
For the AEF dataset, we report results for AEFs evoked by left auditory stimulation with pure tones of 500\,Hz.
The analysis window for source estimation was chosen from \mbox{50\,ms to 200\,ms} based on visual inspection of the evoked data to capture the dominant N100m component, leading to $T=6$.
For the SEF dataset, we analyzed SEFs evoked by bipolar electrical stimulation (0.2\,ms in duration) of the left median nerve.
To capture the main peaks of the evoked response and exclude the strong stimulus artifact, the analysis window was chosen from \mbox{18\,ms to 200\,ms} based on visual inspection of the sensor signal.

Preprocessing was done following the standard pipeline from the MNE
software~\citep{mne}.
Baseline correction using pre-stimulus data
(from -200\,ms to 0\,ms) was used. Epochs with peak-to-peak amplitudes exceeding
predefined rejection parameters (3\,pT for magnetometers and
400\,pT/m for gradiometers, and 150\,$\mu$V for EOG on AEF and VEF and 350\,$\mu$V for SEF) were assumed to be affected by artifacts and discarded.
%
This resulted in 55 (AEF), 67 (SEF) and 111 (SEF) artifact-free measurements
which were average to produce the target matrix $\*Y$.
%
The gain matrix was computed using a set of $p=7498$ cortical locations,
and a three-layer boundary element model.

\section{Related Work}
\label{sec:app_rel_work}

The topic of high-dimensional inference has been addressed in many
recent works. Yet, to the best of our knowledge, none of this literature has
been applied to the source localization problem we consider here.
\begin{itemize}
\item The idea of associating clustering with high-dimensional inference can also be found in recent works with application to genetic data: \citet{Buhlmann:2013} has used a fixed clustering step, which is made adaptive in \cite{Mandozzi:2016}. Our contribution deviates from these works in two regards: unlike \cite{Buhlmann:2013}, we do not consider that a fixed clustering, however good it is, indeed captures the essence of the problem: this is why we resort to an ensemble of different clustering solutions.
Unlike \cite{Mandozzi:2016}, we do not try to narrow down the inference in a hierarchical fashion, because we do not consider that source imaging can in effect be traced down to the vertex level: given the difficulty of the source imaging problem, we find it more satisfactory to outline a region of putative activity.
\item Another family of inference methods based on sample splits has been introduced by \citet{Meinshausen_Meier_Buhlmann09}: train data are used to select regions, test data to assess their statistical significance.
The choice of splits can be varied and aggregated upon to mitigate the impact of arbitrary splits selection.
However, data splitting has a high cost in terms of statistical power, making these approaches weakly sensitive \cite{Taylor:2015}.
\item An alternative method yielding family-wise error rate (FWER) control is the stability selection method, that builds on bootstrapped randomized sparse regression \cite{Meinshausen2010}.
  Yet, this approach has been found too weakly sensitive 
  and it has not been considered in further statistical inference works, see \eg \citet{Dezeure_Buhlmann_Meier_Meinshausen15}.
\item Post-selection inference \cite{Taylor:2015} is an approach that typically relies on a sparse estimator (such as Lasso) and then assesses the significance of the selected variables.
It accounts for the selection in the inference process, avoiding the undesirable bias of selecting and testing on the same data.
However, we have not not found
an implementation that scales in a numerically sound way to the problem size that we are considering here: thousand features, even after clustering.
\item Knockoff inference (with or without clustering) is probably the most recent alternative developed for high-dimensional inference \citep{Barber:2016}: it consists in appending noisy copy of the problem features and selecting only variables that are much more significantly associated than their noisy copy.
While this approach is computationally relevant for the problem at hand, it suffers from the arbitrary knockoff variable set used; it yields a control of the false discovery rate of the detection problem, that is not directly comparable with the family-wise error rate (FWER) considered here.
FWER control is possible with knockoff \cite{janson:2015}, yet very weakly sensitive.
\end{itemize}

\newpage

\section{Supplementary figures}
\label{sec:supp_fig}

\begin{figure}[!ht]
  \begin{minipage}{0.495\linewidth}
  \centering\includegraphics[width=\textwidth]
  {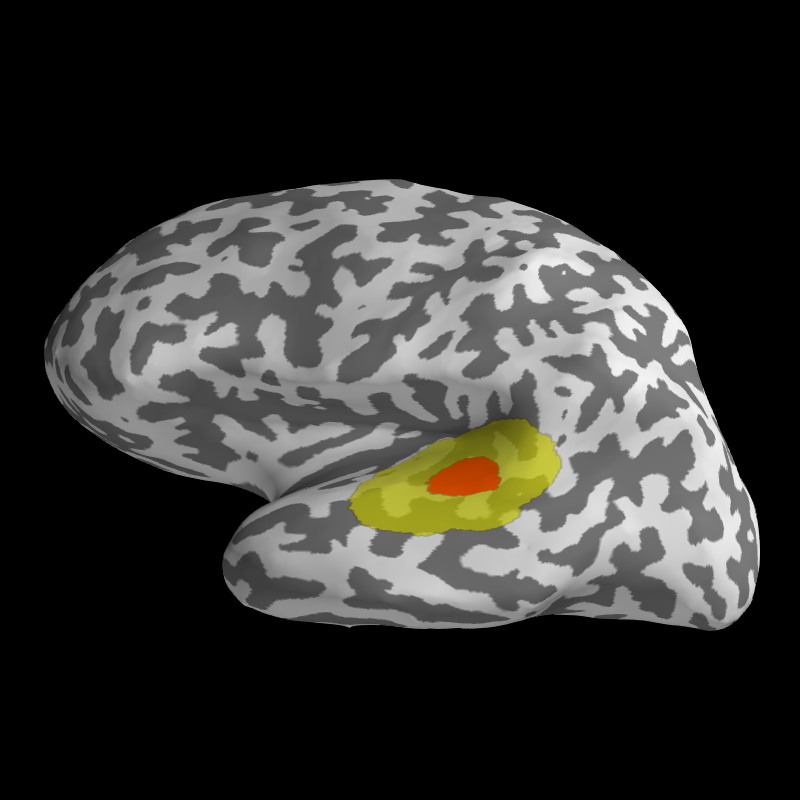}
  \end{minipage}
  \begin{minipage}{0.495\linewidth}
  \centering\includegraphics[width=\textwidth]
  {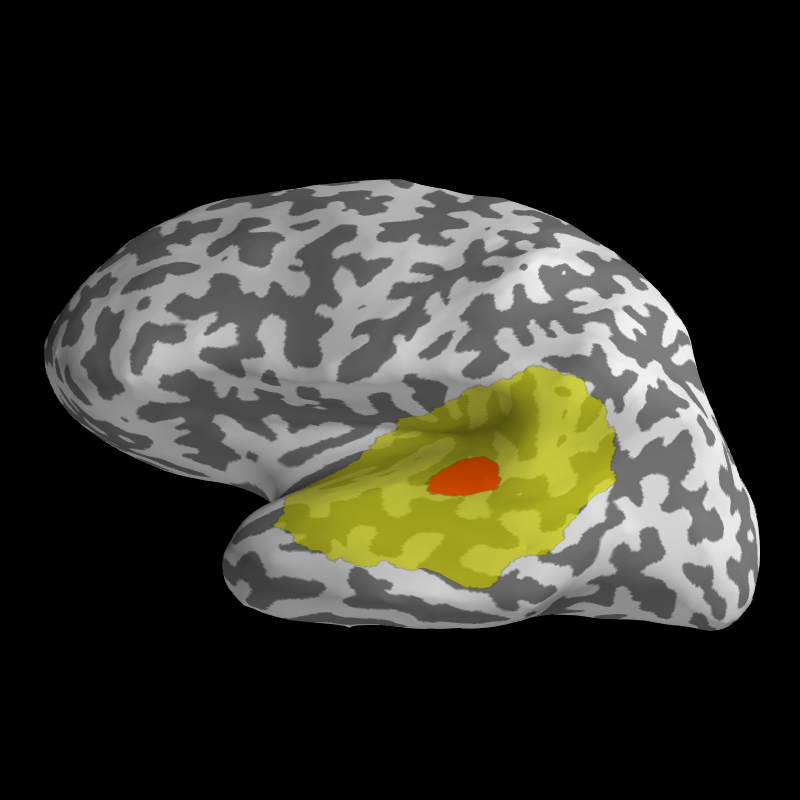}
  \end{minipage}
  \caption{\textbf{Illustrating spatial tolerance of size $\bm\delta$ = 20\,mm
  and $\bm\delta$ = 40\,mm.}
  The true source in red has a 10\,mm radius (distance measured
  on the cortical surface) and the spatial tolerance extend this
  region by 20\,mm on the left side and 40\,mm on the right side
  in yellow.
  The $\delta$-FWER is the probability of making false discoveries outside
  of the extended region. Then, a false discovery made in the yellow region is
  not counted neither as an error nor a true positive.}
  \label{fig:spatial_tolerance}
\end{figure}

\begin{figure}[!ht]
  \begin{minipage}{0.495\linewidth}
    \centering\includegraphics[width=\linewidth]
    {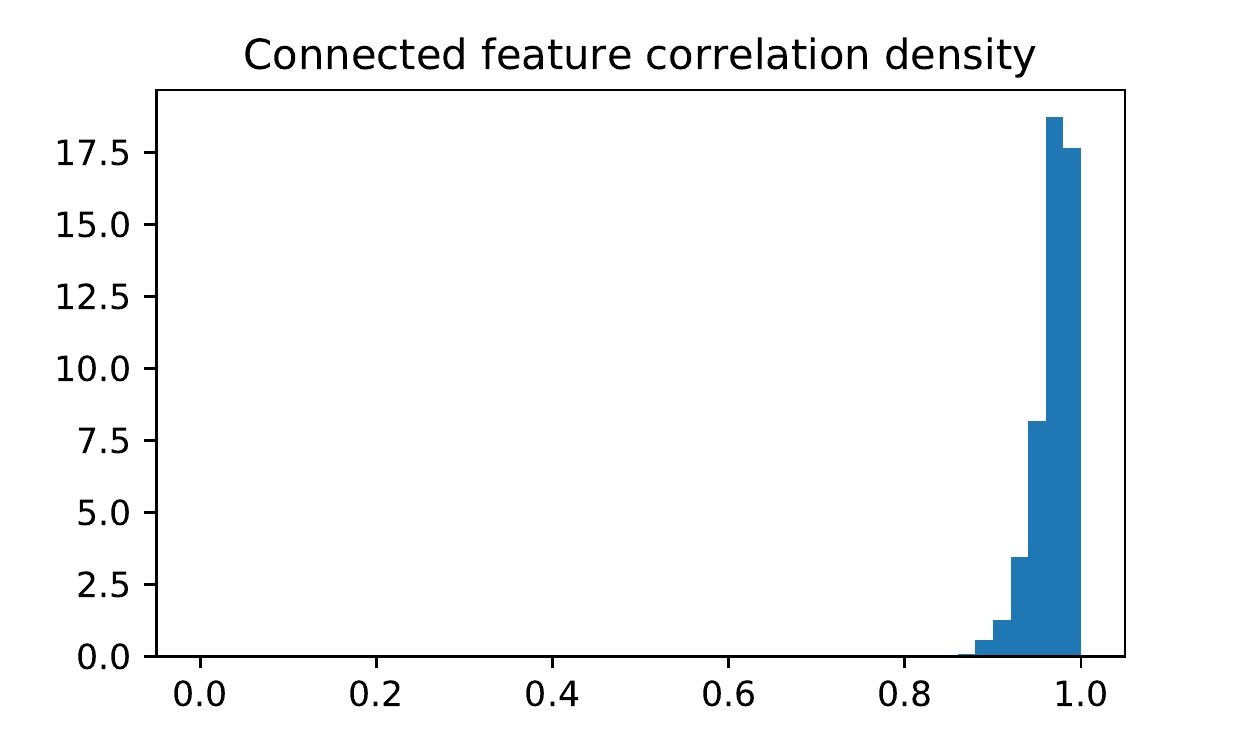}
  \end{minipage}
  \begin{minipage}{0.495\linewidth}
    \centering\includegraphics[width=\linewidth]
    {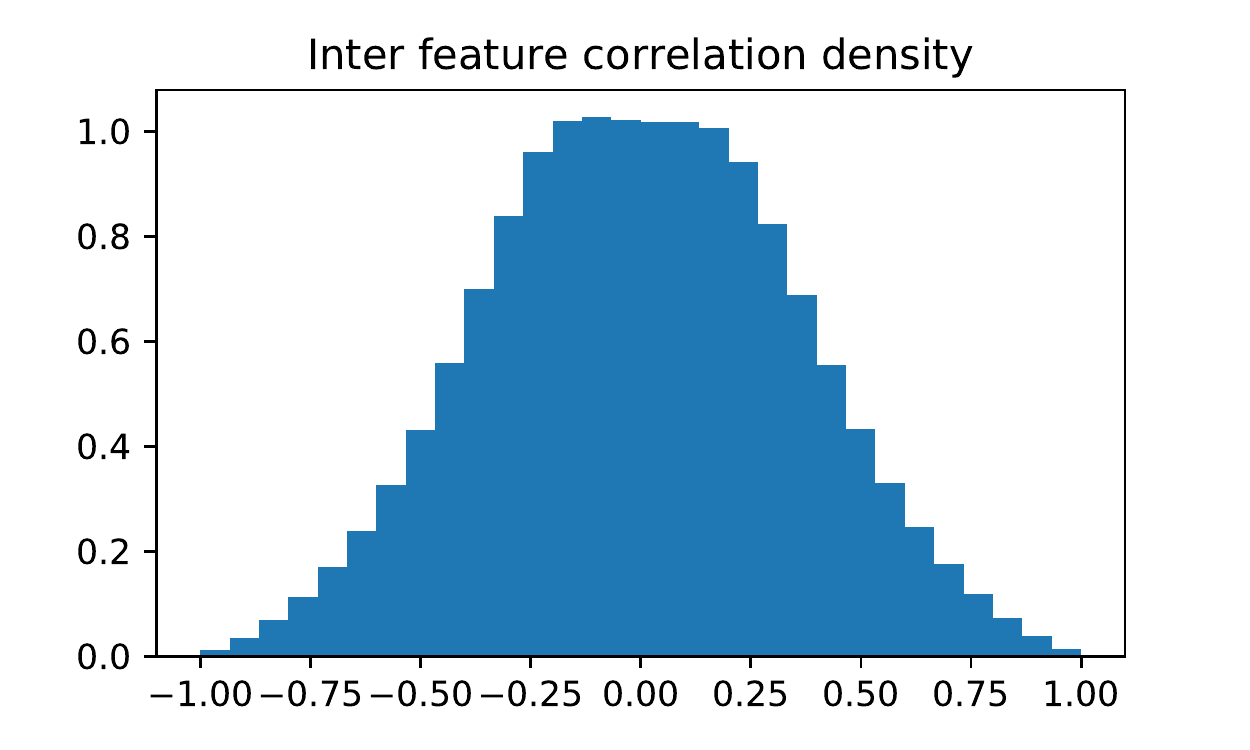}
  \end{minipage}
  \begin{minipage}{0.495\linewidth}
    \centering\includegraphics[width=\linewidth]
    {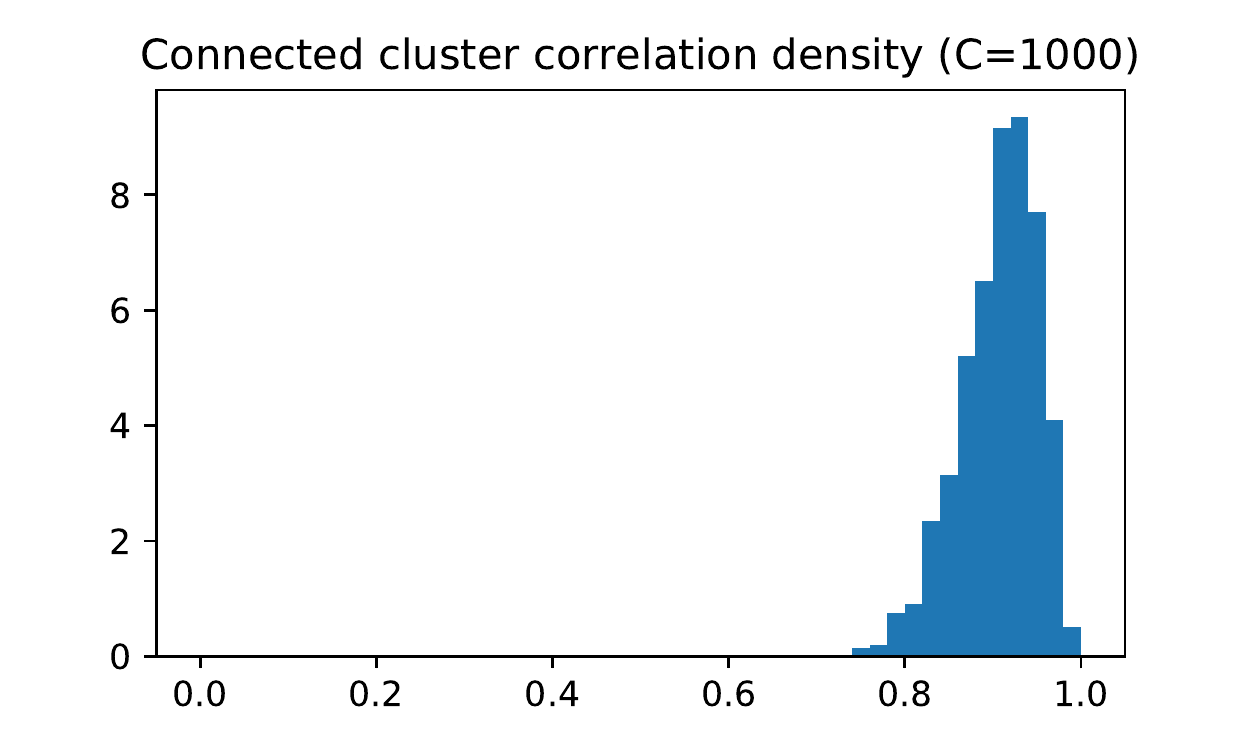}
  \end{minipage}
  \begin{minipage}{0.495\linewidth}
    \centering\includegraphics[width=\linewidth]
    {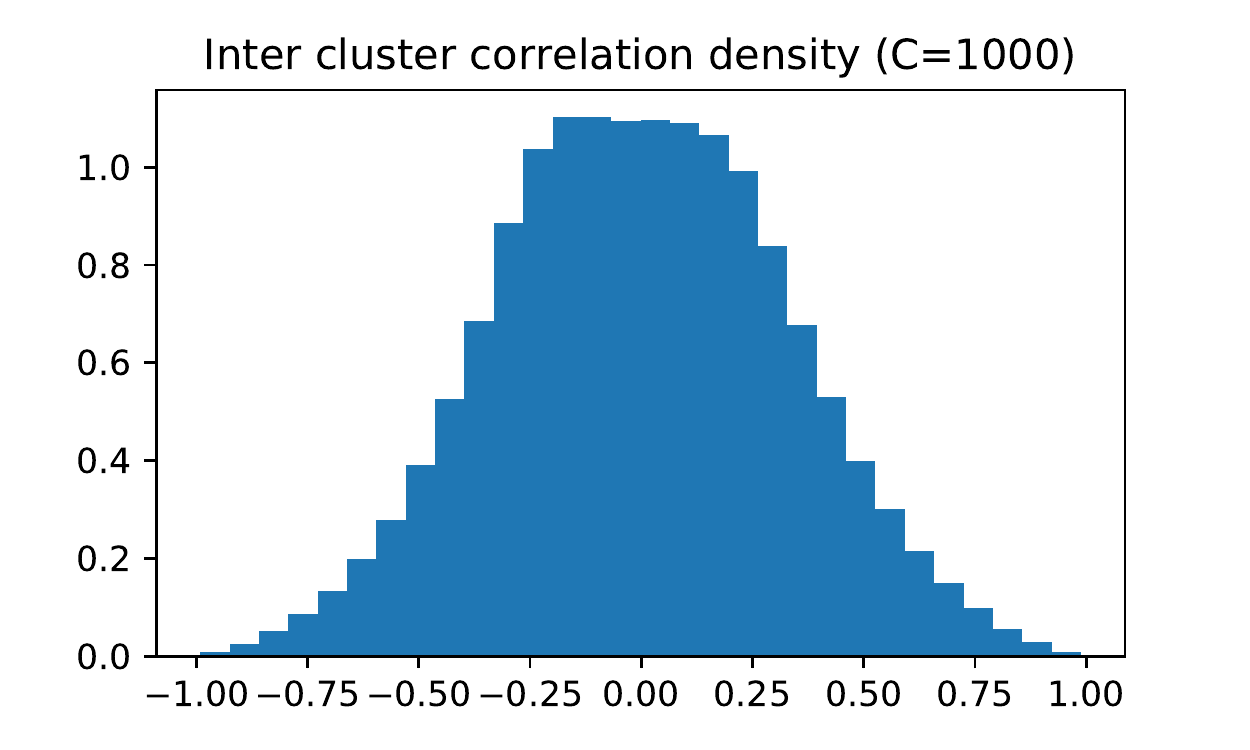}
  \end{minipage} \hfill
  \caption{\textbf{Illustrating correlation in MNE sample MEG data.}
  (left): Distribution of the maximum correlation between a feature
  (resp. cluster) and another connected feature (resp. cluster).
  (Top) the maximum connected feature correlation is close to 0.98 in average. (Bottom) the maximum connected cluster
  correlation is lower, close to 0.9 on average. Clustering improves
  conditioning significantly.
  (right): The density of the inter feature correlation (top) looks
  similar to the density of the inter cluster correlation (bottom).
  By focusing the extreme values of correlation, we see a little
  decrease of extreme values for the clustered data.
  }
  \label{fig:correlation_data}
\end{figure}

\begin{figure}[!ht]
  \centering
  \includegraphics[width=\textwidth]{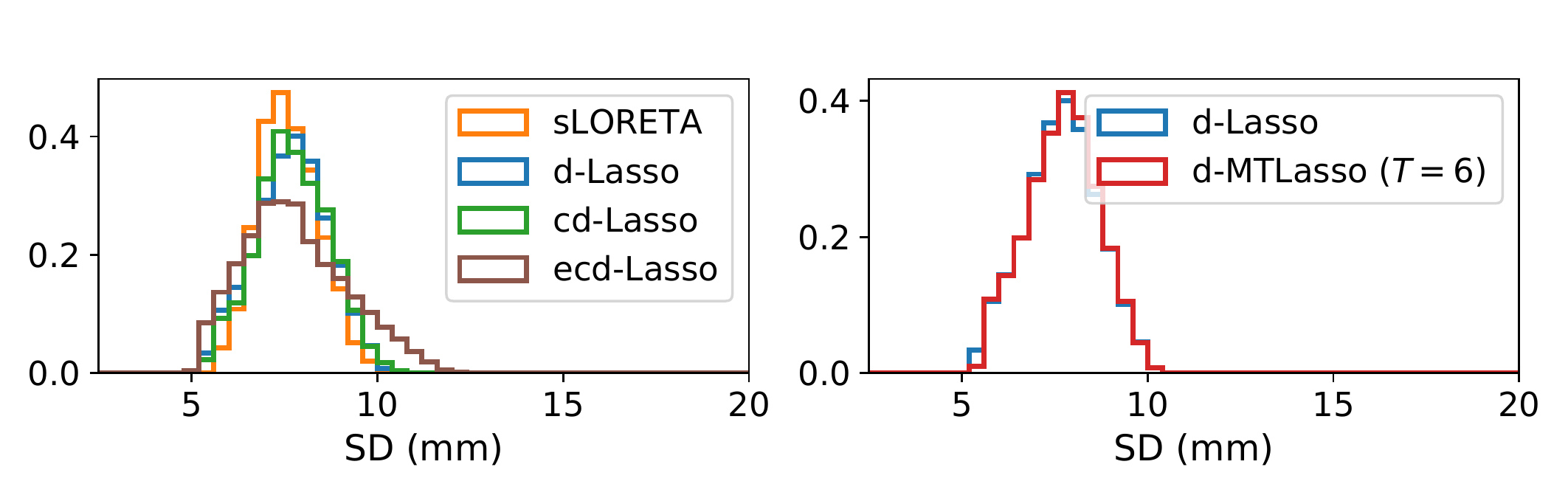}
  \caption{\textbf{Spatial Dispersion (SD) histograms.}
  (left): SD on a fixed time point~\citep{HAUK20111966}. All methods lead to
  comparable spatial dispersion.
  (right): SD for desparsified multi-task Lasso (d-MTLasso) with increasing
  time points. See~\Cref{fig:metrics_hist_sparse_simu} for PLE histograms
  on the same experiments.
  }
  \label{fig:metrics_sd_hist_sparse_simu}
\end{figure}

\begin{figure}[!ht]
  \centering
  \includegraphics[width=0.5\textwidth]{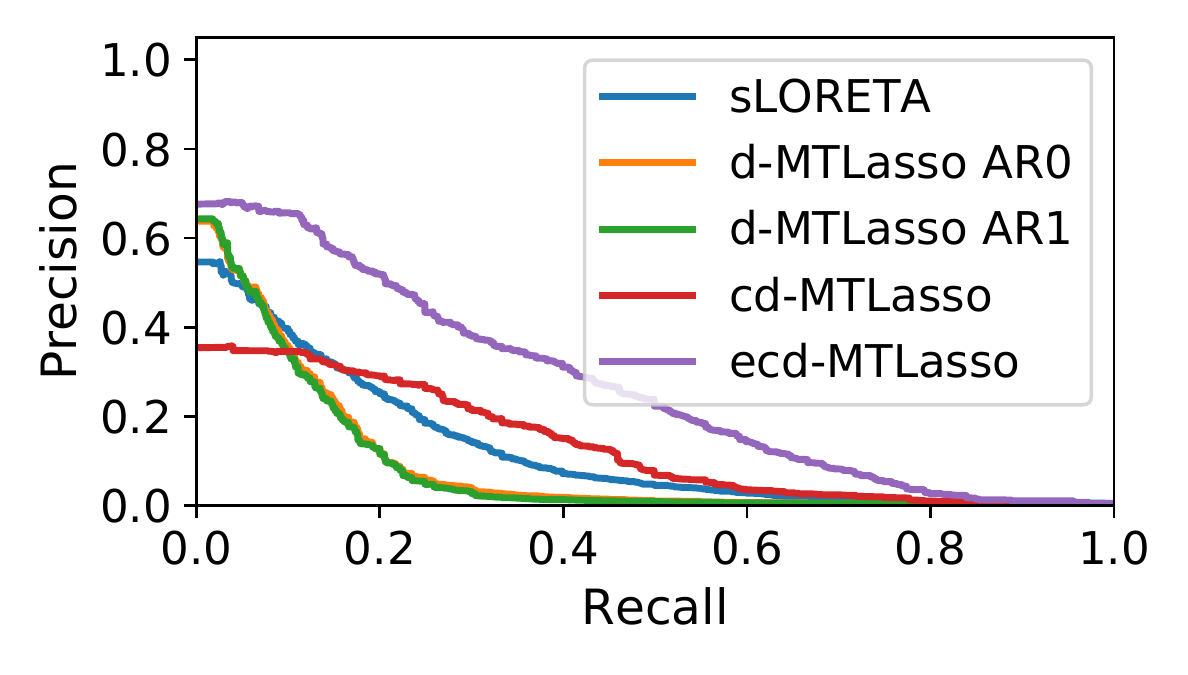}
  \caption{\textbf{Precision-Recall.}
  See~\Cref{fig:fwer_delta_precision_recall} for $\delta-$Precision-Recall curves
  computed on the same data.
  }
  \label{fig:precision_recall}
\end{figure}

\begin{figure}[!ht]
  \centering
  \begin{minipage}{0.495\linewidth}
      \centering\includegraphics[width=\linewidth]
      {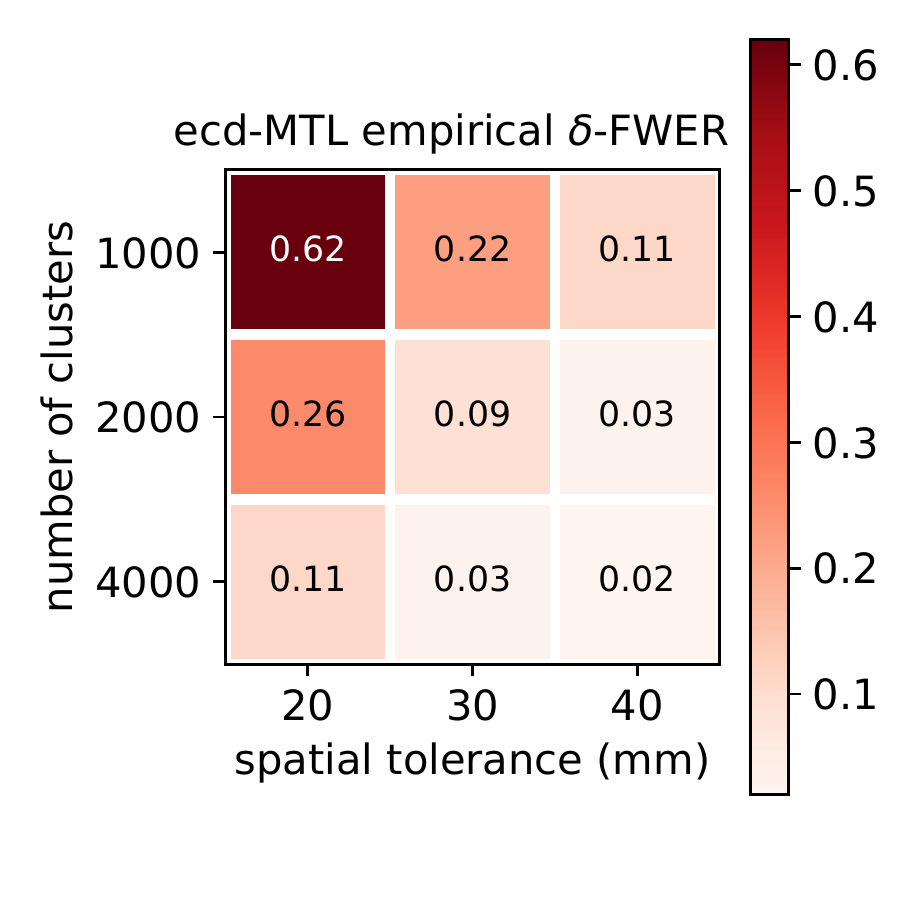}
  \end{minipage}
  \begin{minipage}{0.415\linewidth}
      \centering\includegraphics[width=\linewidth]
      {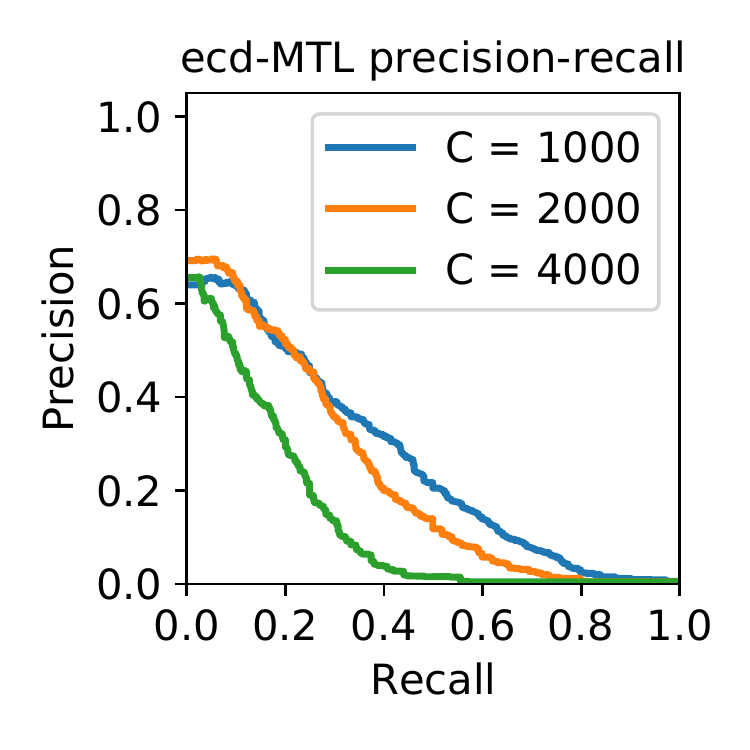}
  \end{minipage}
  \caption{\textbf{ecd-MTLasso empirical $\delta$-FWER and precision recall for
  different choice of cluster sizes.}
  (left): Running the same simulation as in \Cref{subsec:semi_real_simu},
  we observe that the spatial tolerance $\delta$ can be reduced to 20\,mm
  by increasing the number of clusters up to $4000$.
  With $C = 1000$ clusters (resp. $C = 2000$, $C = 4000$), the average cluster
  diameter is around 18\,mm (resp. 13\,mm and 9\,mm).
  It turns out that the  $\delta$-FWER is controlled for around twice the diameter
  (if the compressed design matrix verifies assumption A1).
  (right): We see that this decrease in spatial tolerance comes with a
  price regarding support recovery: the precision-recall curve declines
  with when  $C$ is increased.
  (both): Note that we need to set the hyper-parameter $c$ that is used
  to compute the regularization parameters $\bm\alpha$ (see note coming
  with \Cref{eq:z-def}). We found empirically that it should be inversely proportional to $C$: for $C = 1000$, $c = 0.5 \%$; for $C = 2000$,
  $c = 0.25 \%$; for $C = 4000$, $c = 0.15 \%$.}
  \label{fig:delta-FWER table}
\end{figure}

\begin{figure}[!ht]
  \centering
  \includegraphics[width=0.24\textwidth]{figures/real_audio_sLORETA_lh.png}
  \includegraphics[width=0.24\textwidth]{figures/real_audio_group_inference_AR0_lh.png}
  \includegraphics[width=0.24\textwidth]{figures/real_audio_group_inference_AR1_lh.png}
  \includegraphics[width=0.24\textwidth]{figures/real_audio_ensemble_clustered_group_inference_C=1000_lh.png}
  \includegraphics[width=0.24\textwidth]{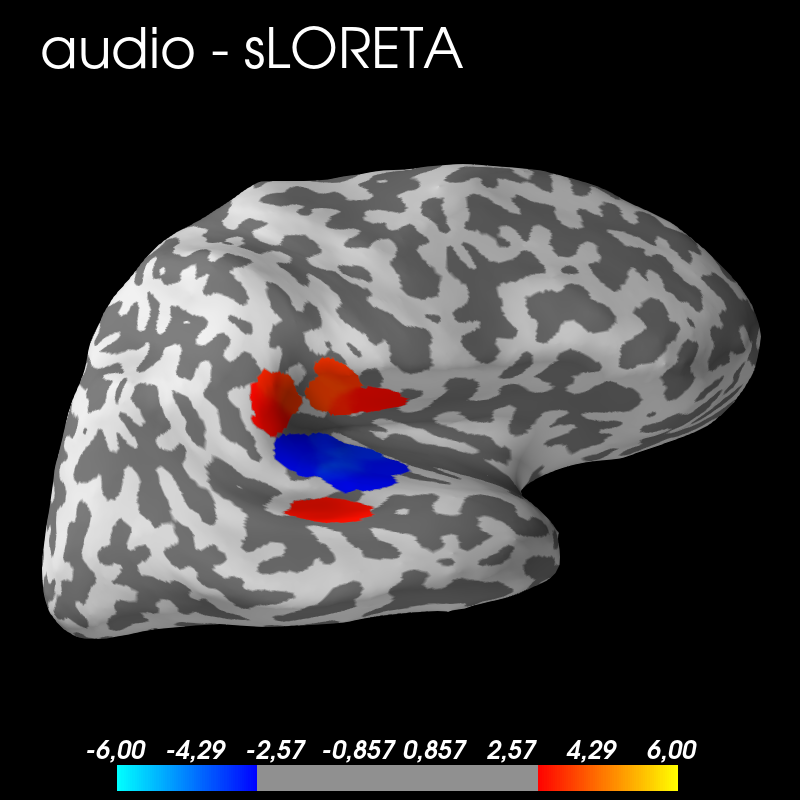}
  \includegraphics[width=0.24\textwidth]{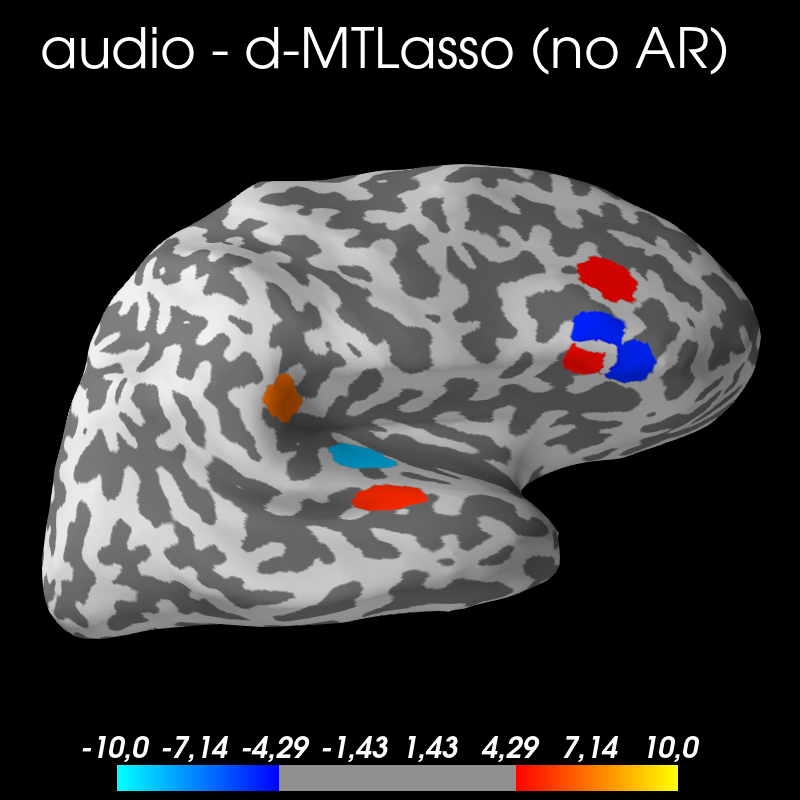}
  \includegraphics[width=0.24\textwidth]{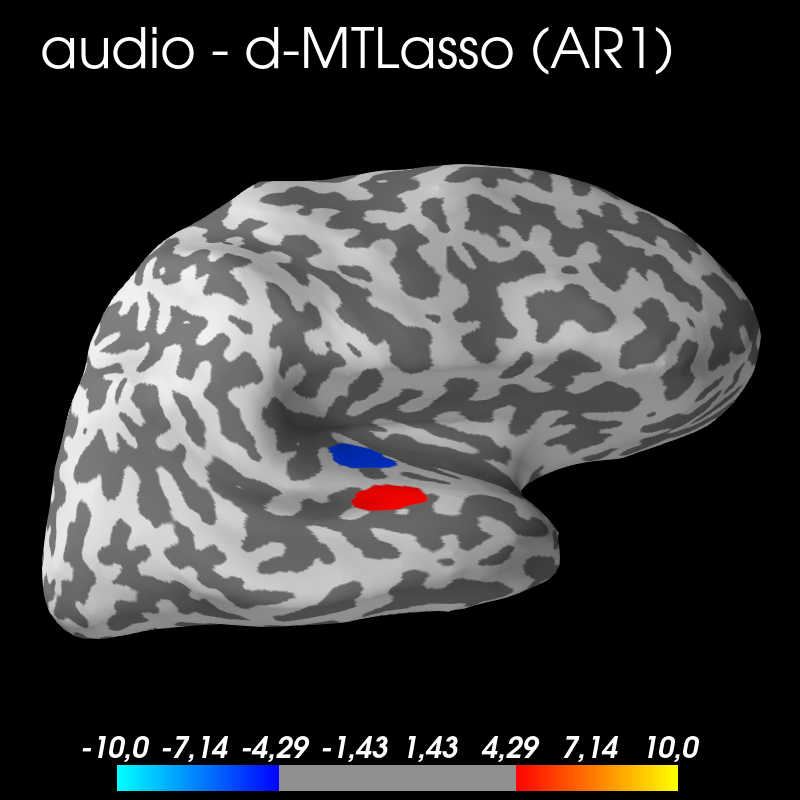}
  \includegraphics[width=0.24\textwidth]{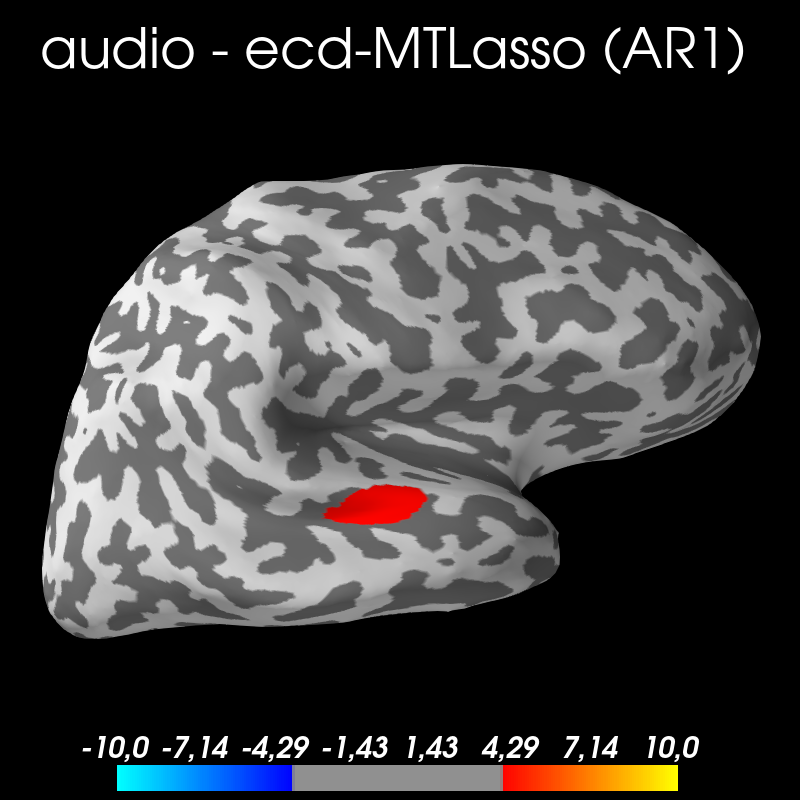}
  \caption{\textbf{Comparison on audio dataset on both hemispheres.}
  From left to right are compared sLORETA, d-MTLasso without AR modeling
  (noise is assumed non-autocorrelated), d-MTLasso with an AR1 noise model
  and the ecd-MTLasso using also an AR1.
  The results correspond to auditory (top) evoked fields.
  Colormaps are fixed across datasets and adjusted based on meaningful
  statistical thresholds in order to outline FWER control issues.
  }
  \label{fig:audio_comparison}
\end{figure}

\begin{figure}[!ht]
  \centering
  \includegraphics[width=0.325\textwidth]
  {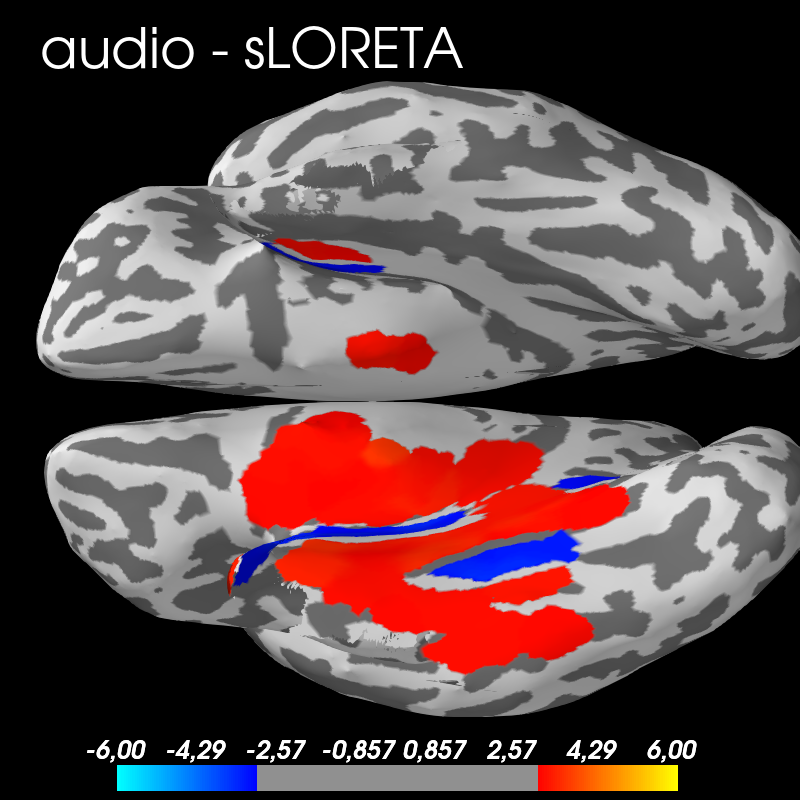}
  \includegraphics[width=0.325\textwidth]
  {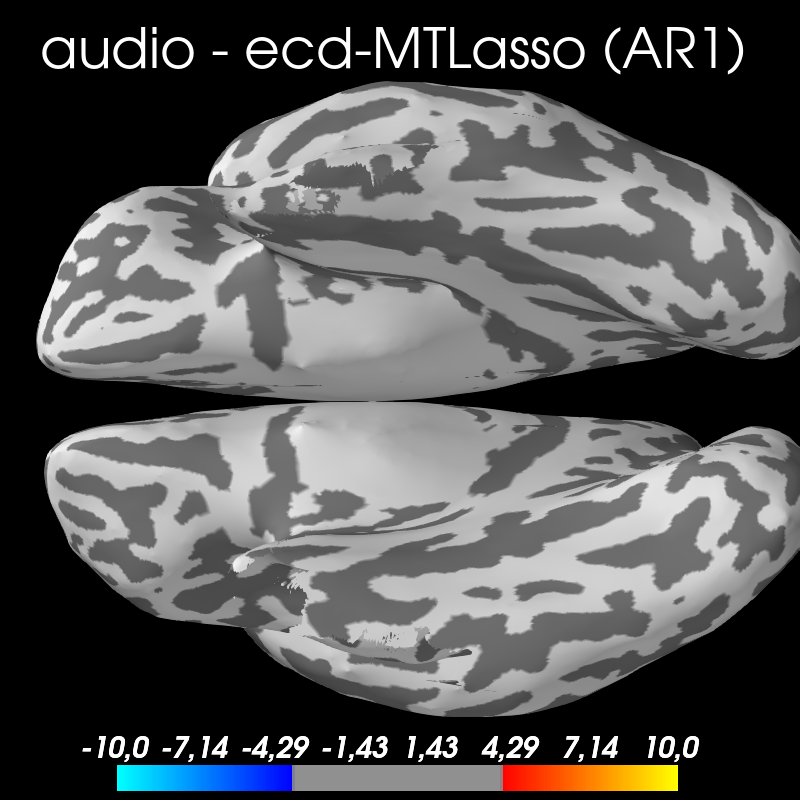}
  \includegraphics[width=0.325\textwidth]
  {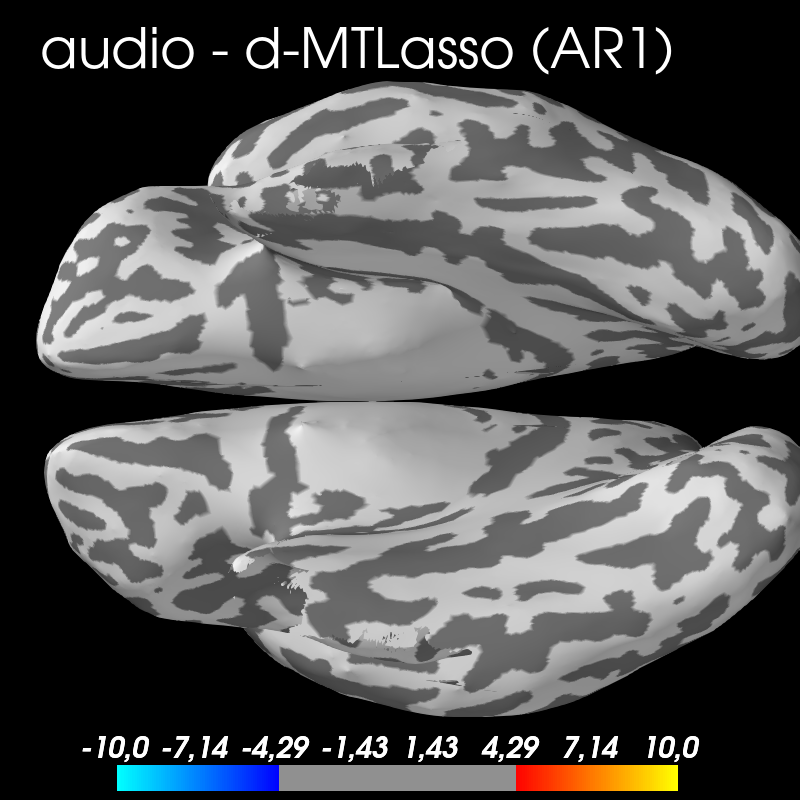}
  \hfill
  \includegraphics[width=0.325\textwidth]
  {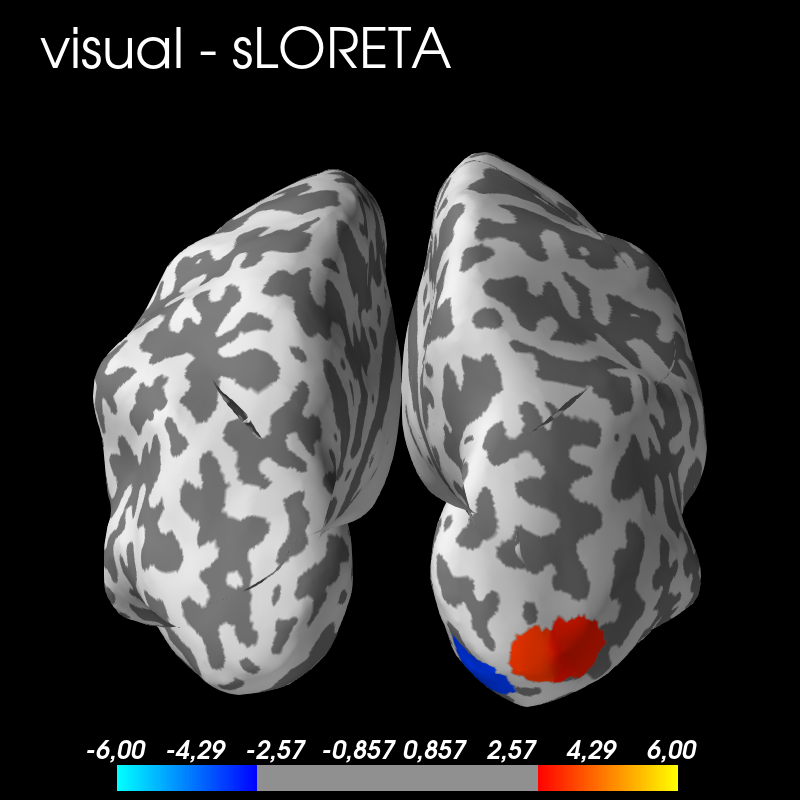}
  \includegraphics[width=0.325\textwidth]
  {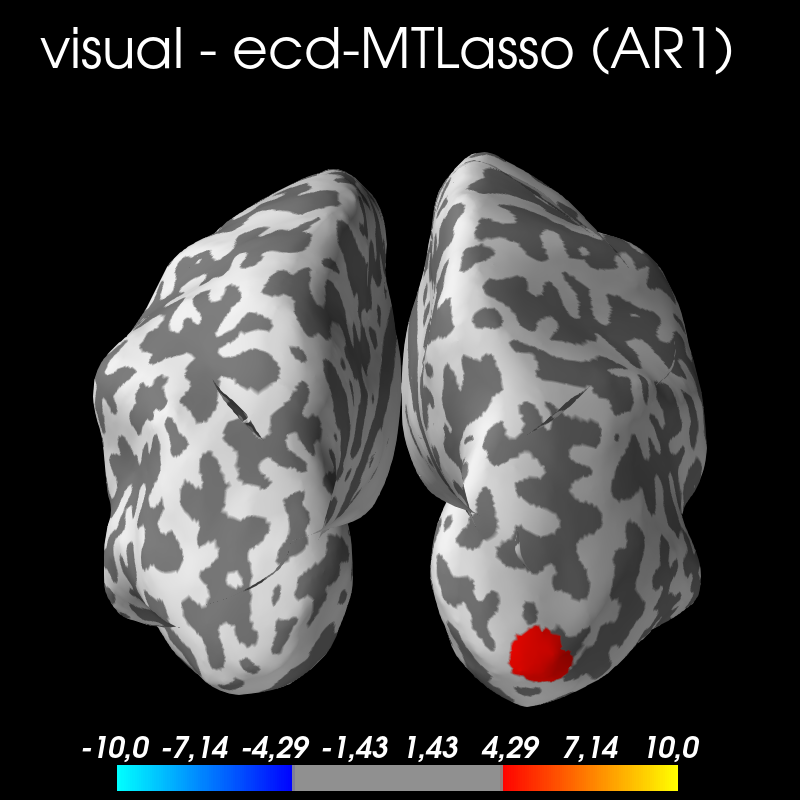}
  \includegraphics[width=0.325\textwidth]
  {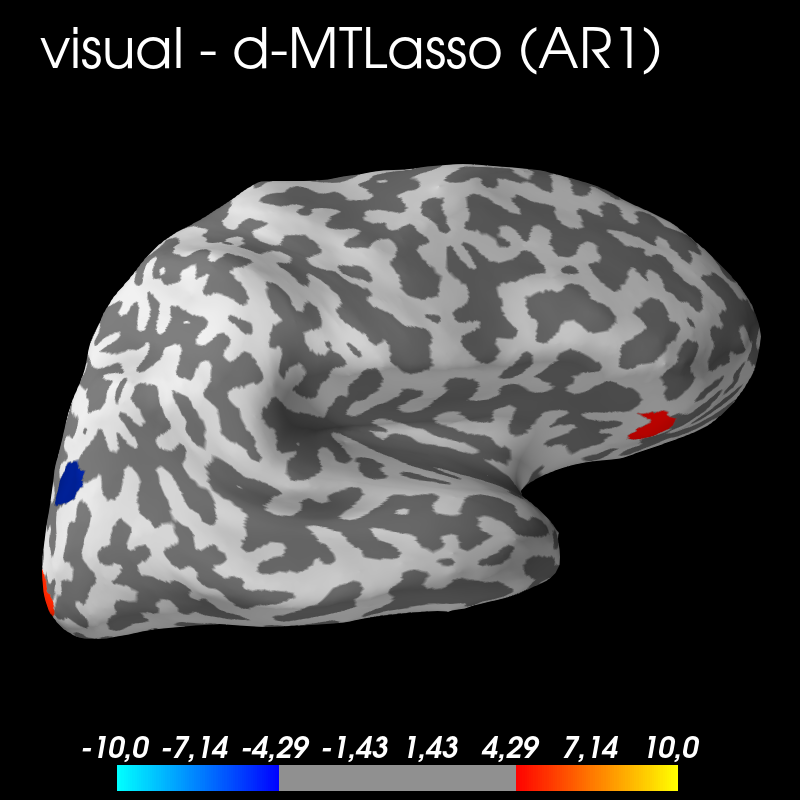}
  \caption{\textbf{Results on real data keeping only EEG sensors.}
  Auditory activations (top) have historically been hard
  to infer with EEG sensors: sLORETA produces only false discoveries
  while ecd-MTL and d-MTL make no discoveries.
  In the visual experiment (bottom): sLORETA and ecd-MTL produce
  expected patterns, d-MTL produces expected patterns plus one
  false discovery in the frontal lobe.
  In our work, we have emphasized MEG experiments:
  they offer more sensors compared to EEG leading to
  improved statistical power.
  }
  \label{fig:EEG_results}
\end{figure}

\end{document}